\documentclass{article}

\usepackage{microtype}
\usepackage{graphicx}
\usepackage{subfigure}
\usepackage{booktabs} 

\usepackage{hyperref}

\usepackage[accepted]{icml2024}

\usepackage{amsmath}
\usepackage{amssymb}
\usepackage{mathtools}
\usepackage{amsthm}
\usepackage{physics}
\usepackage{lipsum}

\newcommand{\R}{\mathbb{R}}

\newcommand{\N}{\mathbb{N}}
 
\newcommand{\E}{\mathbb{E}} 

\newcommand{\argmin}{\mathop{\rm argmin}\limits}
\newcommand{\rcomp}{\mathfrak{R}}
\newcommand{\gcomp}{\mathfrak{G}}
\newcommand{\kl}{\mathop{\rm KL}\nolimits}
\newcommand{\ent}{\mathop{\rm Ent}\nolimits}
\newcommand{\id}{\mathop{\rm Id}}
\newcommand{\ope}{\mathop{\rm op}}

\newcommand{\sgn}{\mathop{\rm sgn}}

\usepackage[capitalize,noabbrev]{cleveref}

\theoremstyle{plain}
\newtheorem{theorem}{Theorem}[section]
\newtheorem{proposition}[theorem]{Proposition}
\newtheorem{lemma}[theorem]{Lemma}

\theoremstyle{definition}
\newtheorem{definition}[theorem]{Definition}
\newtheorem{assumption}[theorem]{Assumption}
\theoremstyle{remark}
\newtheorem{remark}[theorem]{Remark}

\usepackage[textsize=tiny]{todonotes}

\icmltitlerunning{Mean-field Analysis on Two-layer Neural Networks from a Kernel Perspective}

\begin{document}

\twocolumn[
    \icmltitle{Mean-field Analysis on Two-layer Neural Networks from a Kernel Perspective}



    \icmlsetsymbol{equal}{*}

    \begin{icmlauthorlist}
        \icmlauthor{Shokichi Takakura}{todai,riken}
        \icmlauthor{Taiji Suzuki}{todai,riken}
    \end{icmlauthorlist}

    \icmlaffiliation{todai}{Department of Mathematical Informatics, the University of Tokyo, Tokyo, Japan}
    \icmlaffiliation{riken}{Center for Advanced Intelligence Project, RIKEN, Tokyo, Japan}

    \icmlcorrespondingauthor{Shokichi Takakura}{masayoshi361@g.ecc.u-tokyo.ac.jp}

    \icmlkeywords{Machine Learning, ICML}

    \vskip 0.3in
]



\printAffiliationsAndNotice{} 

\begin{abstract}
    In this paper, we study the feature learning ability of two-layer neural networks in the mean-field regime
    through the lens of kernel methods.
    To focus on the dynamics of the kernel induced by the first layer, we utilize a two-timescale limit, where the second layer moves much faster than the first layer.
    In this limit, the learning problem is reduced to the minimization problem over the intrinsic kernel.
    Then, we show the global convergence of the mean-field Langevin dynamics and derive time and particle discretization error.
    We also demonstrate that two-layer neural networks can learn a union of multiple reproducing kernel Hilbert spaces more efficiently than any kernel methods,
    and neural networks aquire data-dependent kernel which aligns with the target function.
    In addition, we develop a label noise procedure, which converges to the global optimum and show that the degrees of freedom appears as an implicit regularization.
\end{abstract}

\section{Introduction}\label{sec:introduction}
Although deep learning has achieved great success in various fields,
the theoretical understanding is still limited.
Several works studied the relation between deep learning and kernel medhods,
which are well-studied in the machine learning community.
A line of work has shown that the training dynamics of infinite-width neural networks can be approximated by linearized dynamics
and the corresponding kernel is called neural tangent kernel (NTK)~\citep{jacot_neural_2018,arora_exact_2019}.
Furthermore, generalizability of neural networks is shown to be characterized by the spectral properties of the NTK~\citep{arora_fine-grained_2019, nitanda_optimal_2020}.
However, the NTK regime is reffered as a lazy regime and
cannot explain the feature learning ability to adapt the intrinsic structure of the data
since neural networks behave as a static kernel machine in the NTK regime.
On the other hand, several works have shown the superiority of the neural networks to the kernel methods in terms of the sample complexity~\citep{barron_universal_1993, yehudai_power_2019, hayakawa_minimax_2020}.
Thus, as shown in several empirical studies~\citep{atanasov_neural_2021, baratin_implicit_2021},
neural networks must acquire the data-dependent kernel by gradient descent.
However, it is challenging to establish a beyond NTK results on the feature learning of neural networks with gradient-based algorithm due to the non-convexity of the optimization landscape.

One promising approach is the mean-field analysis~\citep{mei_mean_2018,hu_mean-field_2020}, which is an infinite-width limit of the neural networks in a different scaling than the NTK regime.
In the mean-field regime, the optimization of 2-layer neural networks, which is non-convex in general,
is reduced to the convex optimization problem over the distribution on the parameters.
Exploiting the convexity of the problem, several works~\citep{nitanda_stochastic_2017, mei_mean_2018,chizat_global_2018} have shown the convergence to the global optimum.
Recently, quantitative optimiztion guarantees has been established for the mean-field Langevin dynamics (MFLD)
which can be regarded as a continuous limit of a noisy gradient descent~\citep{chizat_mean-field_2022, nitanda_convex_2022}.
Moreover, very recently, uniform-in-time results on the particle discretization error have been obtained~\citep{chen_uniform--time_2023,suzuki_uniform--time_2022, suzuki_convergence_2023}.
This allows us to extend results effectively from infinite-width neural networks to finite-width neural networks.

Although the mean-field limit allows us to analyze the feature learning in neural networks,
the connection between mean-field neural networks and its corresponding kernel is still unclear.
To establish the connection to the previous works~\citep{jacot_neural_2018, suzuki_fast_2018,ma_barron_2022} on the relationship between neural networks and kernel methods,
we address the following question:\\
\textit{Is it possible to learn the optimal kernel through the MFLD?
    Futhermore, can this kernel align with the target function by excluding the effect of noise?}

To analyze the dynamics of the kernel inside the neural networks,
we adopt a two-timescale limit~\citep{marion_leveraging_2023}, which separates the dynamics of the first layer and the second layer.
Then, we establish the connection between neural networks training and kernel learning~\citep{bach_multiple_2004}, which involves selecting the optimal kernel for the data.
We provide the global convergence gurantee of the MFLD by showing the convexity of the objective functional
and derive the time and particle discretization error.
Then, we prove that neural networks can aquire \textit{data-dependent} kernel and achieve better sample complexity than any linear estimators including kernel methods for a union of multiple RKHSs.
We also investigate the alignment with the target function and the degrees of freedom of the acquired kernel, which measures the complexity of the kernel,
and develop the label noise procedure which provably reduces the degrees of freedom by just adding the label noise.
Finally, we verify our theoretical findings by numerical experiments.
Our contribution can be summarized as follows:
\begin{itemize}
    \item We prove the convexity of the objective functional with respect to the first layer distribution and the global convergence of the MFLD in two-timescale limit in spite of the complex dependency of the second layer on the distribution of the first layer.
          We also derive the time and particle discretization error of the MFLD.
    \item We show that neural networks can adapt the intrinsic structure of the target function and achieve a better sample complexity than kernel methods for a variant of Baron space~\cite{ma_barron_2022}, which is a union of multiple RKHS.
    \item We study the training dynamics of the kernel induced by the first layer and show that the alignment is increased during the training and achieve $\Omega(1)$ alignment while the kernel alignment at the initialization is $O(1/\sqrt{d})$ for a single-index model, where $d$ is the input dimension.
          We also show that the presence of the intrinsic noise induces a bias towards the large degrees of freedom.
          To alleviate this issue, we propose the label noise procedure to reduce the degrees of freedom and prove the linear convergence to the global optimum.
\end{itemize}
\subsection{Related Works}
\paragraph{Relation between Neural Networks and Kernel Methods}
\citet{suzuki_fast_2018} derived the generalization error bound for deep learning models using the notion of the degrees of freedom in the kernel literature.
\citet{ma_barron_2022} characterized the function class, which two-layer neural networks can approximate, by a union of multiple RKHSs.
However, they did not give any optimization gurantee.
\citet{atanasov_neural_2021} pointed out the connection between training of neural networks and kernel learning,
but their analysis is limited to linear neural networks with whitened data.

\paragraph{Mean-field Analysis}
\citet{chen_generalized_2020} conducted NTK-type analysis using mean-field limit.
However, their analysis relies on the closeness of the first-layer distribution to the initial distribution.
Several works have shown that the superiority of the mean-field neural networks to kernel methods including NTK with global optimization guarantee.
For example,~\citet{suzuki_feature_2023} derived a linear sample complexity with respect to the input dimension $d$ for $k$-sparse parity problems
although they requires exponential time for optimization.
On the other hand, kernel methods require $\Omega(d^k)$ samples.
In addition,~\citet{mahankali_beyond_2023,abbe2022merged} showed the superiority of the mean-field neural networks to the kernel methods for even quartic polymonial.
However, these works fix the second layer during the training to ensure the boundedness of each neuron.
Unlike these studies, we consider the joint training of the first and second layer, and focus on the relationship between neural networks and kernel machines and its implication to the feature learning ability of neural networks.
We remark that~\citet{abbe2022merged} considered two-layer neural networks in the mean-field regime with the learnable second layer and showed the superiority to kernel methods,
but their analysis utilized two-phase training, where the second layer is fixed in phase 1 and the first layer is fixed in phase 2.

\paragraph{Two-timescale Limit}
Two-timescale limit is introduced to the analysis for training of neural networks in \citet{marion_leveraging_2023}.
They provided the global convegence gurantee for simplified neural networks
but their analysis is limited to the single input setting and the relation to the kernel learning was not discussed.
~\citet{bietti_learning_2023} leveraged the two-timescale limit to analyze the training of a non-parametric model for a multi-index model
and show the saddle-to-saddle dynamics.
However, they used non-parametric model for the second layer, which is different from practical settings.
\paragraph{Feature Learning in Two-layer Neural Networks}
Aside from the mean-field analysis,
there exists a line of work which studies the feature learning in (finite-width) two-layer neural networks~\citep{damian_neural_2022, mousavi-hosseini_neural_2022}.
For instance,
~\citet{damian_neural_2022} show that the random feature model with the first layer parameter updated by one-step gradient descent can learn a certain subset of $p$-degree polynomial with $O(d^2)$ samples while kernel methods require $\Omega(d^p)$ samples.
However, most of these works consider two-stage optimization procedure, where the first layer is trained before proceeding to train the second layer.

\paragraph{Implicit Bias of Label Noise}
Implicit bias of label noise has been intensively studied recently~\citep{damian_label_2021, li_what_2021, vivien_label_2022}.
For example,~\citet{li_what_2021} developed a theoretical framework to analyze the implicit bias of noise in small noise and learning rate limit
and prove that the label noise induces bias towards flat minima.
On the other hand, we elucidate the implicit regularization of label noise on the kernel inside the neural networks.

\subsection{Notations}
We write the expectation with respect to $X \sim \mu$ by $\E_{X\sim \mu}[\cdot]$ or $\E_\mu[\cdot]$.
$\kl$ denotes the Kullback-Leibler divergence $\kl(\nu \mid \mu) = \int \log(\frac{\mu(w)}{\nu(w)})\dd \mu(w)$ and $\ent$ denotes the negative entropy $\ent(\mu) = \E_\mu[\log \mu]$.
$N(v, \Sigma)$ denotes the Gaussian distribution with mean $v$ and covariance $\Sigma$ and $\upsilon(S)$ for $S \subset \R^d$ denotes the uniform distribution on $S$.
$\mathcal{P}$ denotes the set of probability measures on $\R^{d'}$ with finite second moment.
For a matrix $A$, $\norm{A}_{\ope}$ denotes the operator norm with respect to $\norm{\cdot}_2$ and $\norm{A}_{\mathrm{F}}$ denotes
the Frobenius norm.
For an operator $A:L^2(\mu) \to L^2(\mu)$, $\norm{A}_{\ope}$ denotes the operator norm with respect to $\norm{\cdot}_{L^2(\mu)}$.
For $l:\R^2 \to \R$, $\partial_1 l$ denotes the partial derivative with respect to the first argument.
For a symmetric matrix $A$, $\lambda_{\min}(A)$ denotes the minimum eigenvalue of $A$.
With a slight abuse of notation, we use $f(X)$ for $f:\R^d \to \R$ and $X = [x^{(1)}, \dots, x^{(n)}]^\top \in \R^{n \times d}$ to denote $[f(x^{(1)}), \dots, f(x^{(n)})]^\top$.

\section{Problem Settings}\label{sec:problem-settings}
\subsection{Mean-field and Two-timescale Limit in Two-layer Neural Networks}
Given input $x \in \R^d$,
let us consider the following two-layer neural network model:
\begin{align*}
    f(x;a, \qty{w_i}_{i=1}^N) = \frac1N\sum_{i=1}^N a_i h(x;w_i),
\end{align*}
where $a_i \in \R,~w_i \in \R^{d'}$ and $h(x;w_i)$ is the activation function with parameter $w_i$.

Mean-field limit of the above model is defined as an integral over neurons:
\begin{align}
    f(x;P) := \int a h(x;w) P(\dd a, \dd w),\label{eq:standard-formulation}
\end{align}
where $P$ is a probability distribution over the parameters of the first and second layers.
However, in this formulation, the first and the second layer are entangled, and thus it is difficult to characterize the feature learning,
which takes place in the first layer.
To alleviate this issue, we consider the following formulation:
\begin{align*}
    f(x;a, \mu) := \int a(w) h(x;w) \dd \mu(w),
\end{align*}
where $a(w) = \int a P(\dd a \mid w)$ and $\mu(w) = \int P(a, w)\dd a$ is the marginal distribution of $w$.
Similar formulation can be found in~\citet{fang_over_2019}.
This formulation explicitly separates the first and the second layer,
which allows us to focus on the feature learning dynamics in the first layer.
More generally, we consider the multi-task learning settings.
That is, $f(x;a, \mu):\R^d \to \R^T$ is defined by
\begin{align*}
    f_i(x;a, \mu) := \int a^{(i)}(w) h(x;w) \dd \mu(w),
\end{align*}
where $a:\R^{d'} \to \R^T$ is the second layer and $a^{(i)}$ is the $i$-th component of $a$.
Note that the first layer $\mu(w)$ is shared among tasks.

Let $\rho$ be the true or empirical distribution of the pair of input and output $(x, y) \in \R^{d + T}$,
and $\rho_X$ be the marginal distribution of $x$.
Then, for $\lambda > 0$, the (regularized) risk is defined by
\begin{align*}
    L(a, \mu) & = \frac{1}{T} \sum_{i=1}^T \E_\rho[l_i(f_i(x;a, \mu), y)], \\
    F(a, \mu) & = L(a, \mu) + \lambda \E_\mu[r(a(w), w)],
\end{align*}
where $l_i$ is the loss function for the $i$-th task, and $r$ is the regularization term.
In this paper, we consider $l^2$-regularization $r(a, w) = \frac{\lambda_a}{2 T}\sum_{i=1}^T {a^{(i)}}^2  + \frac{\lambda_w}{2}\norm{w}_2^2$, where $a = (a^{(i)})_{i=1}^T$ and $\lambda_a, \lambda_w > 0$.
We define $\bar \lambda_a = \lambda \lambda_a$, $\bar \lambda_w = \lambda \lambda_w$, and $\nu = N(0, I / \lambda_w)$ for notational simplicity.

To separate the dynamics of the first and second layer, we introduce the two-timescale limit~\citep{marion_leveraging_2023}, where the second layer moves much faster than the first layer.
In this limit, the first layer $a^{(i)}$ converges instantaneously to the unique optimal solution of $\min_a E_\rho[l(f(x;a, \mu), y)] + \frac{\bar \lambda_a}{2}\norm{a}_{L^2(\mu)}^2$ since $F(a, \mu)$ is strongly convex with respect to $a$,
As shown in the next section, $\norm{a}_{L^2(\mu)}^2$ corresponds to the RKHS norm for the kernel induced by the first layer.
Since the optimal second layer is a functional of the first layer distribution $\mu$, we write $a_\mu$ for the optimal solution.
Then, the learning problem is reduced to the minimization of the limiting functional $G(\mu) = F(a_\mu, \mu)$.
We also define $U(\mu)$ by $U(\mu) := L(a_\mu, \mu) + \frac{\bar \lambda}{2T} \E_\mu[\norm{a_\mu(w)}_2^2]$

Throughout the paper, we assume that $h(x;w_i)$ satisfies Assumption~\ref{assumption:activation}.
For example, $\tanh(u \cdot x + b)(w = (u, b))$ satisfies the assumption
if $\E_{\rho_X}[\norm{x}_2^2], \E_{\rho_X}[\norm{x}_2^4]$ are finite.
\begin{assumption}\label{assumption:activation}
    $h(x;w)$ is twice differentiable with respect to $w$
    and there exist constants $c_R, c_L > 0$ such that $\sup_w \abs{h(x;w)} \leq 1, \E_{\rho_X}[\sup_w \norm{\grad_w h(x;w)}_2^2] \leq c_R^2, \E_{\rho_X}[\sup_w\norm{\grad_w^2 h(x;w)}_{\ope}^2] \leq c_L^2$.
\end{assumption}

\subsection{Kernel Induced by the First Layer}
Let us define the kernel induced by the first layer as follows:
\begin{align*}
    k_\mu(x, x') = \int h(x;w) h(x';w) \dd \mu(w).
\end{align*}
Obviously, this is a symmetric positive definite kernel.
It is well-known that there exists a unique RKHS $\mathcal{H}_\mu$ corresponding the kernel $k_{\mu}$.
Furthermore, the RKHS norm $\norm{f}_{\mathcal{H}_\mu}$ is equal to the minimum of $\norm{a}_{L^2(\mu)}$
over all $a$ such that $f(x;a, \mu) = \int a(w) h(x;w) \dd \mu(w)$~\citep{bach_equivalence_2017}.
Thus, the learning problem of the second layer is equivalent to the following optimization problem:
\begin{align}
    \min_{f_i \in \mathcal{H}_\mu} \sum_{i=1}^T \E_\rho[l_i(f_i(x), y)] + \sum_{i=1}^T \frac{\bar \lambda_a}{2} \norm{f_i}_{\mathcal{H}_\mu}^2.\label{eq:kernel-learning}
\end{align}
Thus, learning first layer is equivalent to kernel learning~\citep{bach_multiple_2004}, which choosing suitable RKHS $\mathcal{H}_\mu$.

\subsection{Mean-field Langevin Dynamics}
We optimize the limiting functional $G(\mu)$ by the mean-field Langevin dynamics (MFLD):
\begin{align*}
    \dd w_t & = -\grad_w \fdv{G(\mu)}{\mu} \dd t + \sqrt{2 \lambda} \dd B_t,
\end{align*}
where $w_0 \sim \mu_0 := \nu$, $(B_t)_{t \geq 0}$ is the $d'$-dimensional Brownian motion, and $\fdv{G(\mu)}{\mu}$ is the first variation of $G(\mu)$.
The Fokker-Planck equation of the above SDE is given by
\begin{align*}
    \partial_t \mu_t = \lambda \Delta \mu_t + \grad \cdot \qty[\mu_t \grad \fdv{G(\mu)}{\mu}],
\end{align*}
where $\mu_t$ is the distribution of $w_t$.
It is known that the MFLD is a Wasserstein gradient flow which minimizes the entropy-regularized functional: $\mathcal{G}(\mu) := G(\mu) + \lambda \ent(\mu)$.

To implement the MFLD, we need time and particle discretization~\citep{chen_uniform--time_2023,suzuki_uniform--time_2022, suzuki_convergence_2023}.
For a set of $N$ particles $W = \{w_i\}_{i=1}^N$, we define the empirical distribution $\mu_W = \frac{1}{N} \sum_{i=1}^N \delta_{w_i}$.
Let $W_k = \{w_i^{(k)}\}_{i=1}^N$ be a set of $N$ particles at the $k$-th iteration, and $\mu_k^{(N)}$ be a distribution of $W_k$ on $\R^{d' \times N}$.
Then, at each step, we update the particles as follows:
\begin{align*}
    w^{(k + 1)}_i = w^{(k)}_i - \eta \grad \fdv{G(\mu_{W_k})}{\mu}\/(w^{(k)}_i) + \sqrt{2 \eta \lambda} \xi^{(k)}_i,
\end{align*}
where $\eta > 0$ is the step size and $\xi^{(k)}_i \sim N(0, I)$ are i.i.d. Gaussian noise.
This can be regarded as a noisy gradient descent.
For more detailed discussion, see~\citet{hu2019mean, suzuki_convergence_2023} for example.

\section{Convergence Analysis}\label{sec:convergence-analysis}
As shown in \citet{nitanda_convex_2022,chizat_mean-field_2022}, the convergence of the MFLD depends on the convexity of the functional and the properties of the proximal Gibbs distribution, which is defined as
$
    p_\mu(w) \propto \exp\qty(-\frac{1}{\lambda} \fdv{G(\mu)}{\mu}\/(w)).
$
In the training of neural networks, the convexity is usually ensured by the linearity of $f$ with respect to distribution as in Eq.~\eqref{eq:standard-formulation}.
On the other hand, in the two-timescale limit, $f(x;\mu):=f(x;a_\mu, \mu)$ is not linear with respect to $\mu$
because the second layers $a_\mu$ depend on $\mu$ in a non-linear way.
However, we can prove that the functional $G(\mu)$ is convex and its first variation can be written in a simple form if $\{l_i\}_{i=1}^T$ are convex.

\begin{theorem}\label{thm:convexity}
    Assume that the losses $\qty{l_i}_{i=1}^T$ are convex.
    Then, the limiting functional $G(\mu)$ is convex.
    That is, it holds that
    \begin{align*}
        G(\mu_1) + \int \fdv{G(\mu_1)}{\mu}\/(w) (\mu_2(w) - \mu_1(w)) \dd w \leq G(\mu_2)
    \end{align*}
    for any $\mu_1, \mu_2 \in \mathcal{P}$.
    In addition, the first variation of $G(\mu)$ is given by
    \begin{align}
        \fdv{G}{\mu}\/(\mu)(w) 
         & = \lambda \qty(-\frac{\lambda_a}{2T} \norm{a_\mu(w)}_2^2 + \frac{\lambda_w}{2} \norm{w}_2^2). \label{eq:variation-G-l2}
    \end{align}
\end{theorem}
See Appendix~\ref{sec:proof-convexity} for the proof.
We remark that the convexity holds for general regularization term $r$ which is strongly convex with respect to $a$.

The convergence rate of the MFLD depends on the constant in the log-Sobolev inequality for the proximal distribution $p_\mu$.
\begin{definition}
    We say a probability distribution $\mu$ satisfies log-Sobolev inequality with constant $\alpha > 0$
    if for all smooth function $g:\R^d \to \R$ with $\E_\mu[g^2] < \infty$,
    \begin{align*}
        \E_\mu[g^2 \log g^2] - \E_\mu[g^2]\log \E_\mu[g^2] \leq \frac{2}{\alpha} \E_\mu[\norm{\nabla g}^2].
    \end{align*}
\end{definition}

To derive the LSI constant  $\alpha$, we assume either of the following conditions on the loss functions:
\begin{assumption}\label{assumption:l2-loss}
    $l_i$ is squared loss. That is, $l_i(z, y) = \frac{1}{2}(z - y)^2$.
    In addition, $\abs{y_i} \leq c_l$ a.s. for some constant $c_l > 0$.
\end{assumption}
\begin{assumption}\label{assumption:convex-loss}
    $l_i$ is convex and twice differentiable with respect to the first argument and $\abs{\partial_1 l_i(z, y)} \leq c_l, \abs{\partial_1^2 l_i(z, y)} \leq 1$ for any $z, y \in \R$.
\end{assumption}
The latter assumption is satisfied by several loss functions such as logistic loss.
Then, using the formula for the first variation, we can derive the LSI constant applying the Holley-Stroock argument~\cite{holley_logarithmic_1987} for bounded perturbation.
\begin{lemma}\label{lem:lsi}
    Assume that each $l_i$ satisfies Assumption~\ref{assumption:l2-loss} or~\ref{assumption:convex-loss}.
    Then, the proximal distribution $p_\mu$ for any $\mu \in \mathcal{P}$ satisfies LSI with constant $\alpha = \lambda_w \exp(-2\frac{\lambda_a c_l^2}{\bar \lambda_a^2})$.
\end{lemma}
The proof can be found in Appendix~\ref{sec:proof-lsi}.
Unfortunately, the LSI constant is extremely small if $\lambda$ is small,
which reads to exponential computational complexity with respect to $1/\lambda$.
We remark that similar dependency also appears in some previous works~\cite{chizat_mean-field_2022, nitanda_convex_2022, suzuki_convergence_2023},
which consider models with the fixed second layer.

\begin{remark}\label{rem:two-timescale}
    In the standard formulation~\eqref{eq:standard-formulation},
    the first variation of $F(P) = \E_\rho[l(f(x;P), y)] + \lambda \E_P[r(a, w)]$ is given by $\fdv{F}{P}(w) = \E_\rho[\partial_1 l(f(x;P), y) ah(x;w)] + \lambda r(a, w)$.
    Then, $\E_\rho[\partial_1 l(f(x;P), y)ah(x;w)]$ is not bounded nor Lipschitz continuous with respect to $(a, w)$,
    even if $\partial_1 l$ and $h$ is bounded.
    Therefore, without two-timescale limit, it is difficult to obtain a LSI constant even for the single output setting.
    Indeed, previous works fix the second layer or clip the output using some bounded function~\cite{chizat_mean-field_2022,nitanda_convex_2022, suzuki_convergence_2023} to ensure the boundedness or Lipschitz continuity of the output of neurons.
\end{remark}

Combining above results, we can show the linear convergence of the MFLD.
\begin{theorem}\label{thm:convergence}
    Let $\mu^*$ be the minimizer of $\mathcal{G}(\mu)$ and $\mathcal{G}^N(\mu^{(N)}_k) = N\E_{W \sim \mu^{(N)}_k}[G(\mu_{W})] + \lambda \ent(\mu^{(N)}_k)$.
    Then, for the constant $\alpha$ in Lemma~\ref{lem:lsi}, $\mu_t$ satisfies
    \begin{align*}
        \mathcal{G}(\mu_t) - \mathcal{G}(\mu^*) \leq \exp(-2\alpha \lambda t)(\mathcal{G}(\mu_0) - \mathcal{G}(\mu^*))
    \end{align*}
    for any $0\leq t$.
    Furthermore, for any $\eta < \min\qty{1/4, 1/(4\lambda\alpha)}$, $\mu^{(N)}_k$ satisfies
    \begin{align*}
         & \frac{1}{N} \mathcal{G}^N(\mu^{(N)}_k) - \mathcal{G}(\mu^*)                                                           \\
         & \quad \leq \exp(-\alpha \lambda \eta k / 2)(\mathcal{G}^N(\mu^{(N)}_0) - \mathcal{G}(\mu^*)) + \bar \delta_{\eta, N},
    \end{align*}
    where $\bar \delta_{\eta, N} = O(\frac{1}{N} + \eta)$.
\end{theorem}
See Appendix~\ref{sec:proof-convergence} for the proof and the concrete expression of discretization error $\bar \delta_{\eta, N}$.
The proof is based on the framework in~\citet{suzuki_convergence_2023}.
To obtain discretization error, we prove some additional conditions on the smoothness of the objective functional via the optimality condition on $a_\mu$.
This is far from trivial due to the non-linear dependency of $a_\mu$ on $\mu$.
Note that we cannot apply the arguments in~\citet{chizat_mean-field_2022,nitanda_convex_2022,suzuki_convergence_2023} without two-timescale limit
since they assume the boundedness or Lipschitz continuity of each neuron.
See also Remark~\ref{rem:two-timescale} for detailed discussion.

Furthermore, the convergence of the loss function in the discretized setting can be transferred to the convergence of the function value of the neural networks
as shown in the following proposition.
\begin{proposition}\label{prop:convergence-function}
    Assume that $h(x;w)$ is $c_R$-Lipschitz continuous with respect to $w$
    for any $x \in S$, where $S$ is some subset of $\R^d$.
    Let $\Delta = \frac{c_R^2}{\lambda \alpha} (\mathcal{G}^N(\mu_k^{(N)}) - N\mathcal{G}(\mu^*)) + \frac{c_R^2\mathcal{G}(\mu_0)}{\bar \lambda_w}$.
    Then, we have
    \begin{align*}
        \E_{W_k \sim \mu_k^{(N)}}\qty[\sup_{\substack{(x, y) \\ \in S \times S}}\abs{k_{\mu_{W_k}}(x, y) - k_{\mu^*}(x, y)}^2] & = O\qty(\frac{\Delta}{N}).
    \end{align*}
    In addition, if $l_i$ satisfies Assumption~\ref{assumption:l2-loss}, then we have
    \begin{align*}
         & \E_{W_k \sim \mu_k^{(N)}}\qty[\sup_{x \in S}(f_i(x;\mu_{W_k}) - f_i(x;\mu^*))^2]            \\
         & \quad = O\qty(\frac{c_0^2(\bar \lambda_a^2 + 1)}{\bar \lambda_a^4} \cdot \frac{\Delta}{N}).
    \end{align*}
\end{proposition}
See Appendix~\ref{sec:proof-convergence-function} for the proof.
Note that this result is not covered by Lemma 2 in~\citet{suzuki_feature_2023} since their analysis relies on the Lipschitz continuity of each neuron.
In the following sections, we consider infinite-width neural networks trained by the MFLD for simplicity,
but the results can be transferred  via this proposition to the finite-width neural networks trained by the discretized MFLD.

\section{Generalization Error for Barron Spaces and Superiority to Kernel Methods}\label{sec:generalization-error}
In this section, we provide the separation of the generalization error between neural networks and kernel methods which cannot adapt the intrinsic structure of the target function.

Let $\mathcal{D} = \qty{(x^{(i)}, y^{(i)})}_{i=1}^n$ be training data sampled from the true distribution in an i.i.d. manner.
We define $X = \qty[x^{(1)}, \dots, x^{(n)}]^\top \in \R^{n \times d}$, $Y_i = [y^{(1)}_i, \dots, y^{(n)}_i]^\top \in \R^n$ and $\hat \Sigma_\mu = \E_\mu[h(X;w)h(X;w)^\top]$.
In the following, we write the true distribution by $\rho$ and the empirical distribution by $\hat \rho$.
In addition, to distinguish the empirical risk and population risk, we write $U_\rho(\mu), \mathcal{G}_\rho(\mu)$ for the (regularized) population risk and $U_{\hat \rho}(\mu), \mathcal{G}_{\hat \rho}(\mu)$ for the empirical risk.
Then, we assume the following.
\begin{assumption}\label{assumption:data}
    the output $y_i$ for each task is generated by $y_i = f_i^\circ(x) + \varepsilon_i$,
    where $f_i^\circ:\R^d \to \R$ is the target function and $\varepsilon_i$ is the noise, which follows $\upsilon([-\sigma, \sigma])$ independently for some $\sigma \geq 0$.
\end{assumption}

To see the benefit of the feature learning or kernel learning, we consider the following function class.
\begin{definition}[KL-restricted Barron space]
    Let $\mathcal{P}_M = \qty{\mu \in \mathcal{P} \mid \kl(\nu \mid \mu) \leq M}$ for some $M > 0$.
    Then, we define the KL-restricted Barron space as
    \begin{align*}
        \mathcal{B}_M = \qty{f(x;a, \mu) \mid \mu \in \mathcal{P}_M, a \in L^2(\mu)},
    \end{align*}
    and the corresponding norm as
    \begin{align*}
        \norm{g}_{\mathcal{B}_M} = \inf_{\mu \in \mathcal{P}_M, a \in L^2(\mu)} \qty{\norm{a}_{L^2(\mu)} \mid g(x) = f(x;a, \mu)}.
    \end{align*}
\end{definition}
This can be seen as a variant of Barron space in~\citet{e_priori_2019, ma_barron_2022}.
Similar function classes are also considered in~\citet{bach_equivalence_2017} but they consider Frank-Wolfe type optimization algorithm, which is different from usual gradient descent.
We remark that Barron space can be regarded as a union of RKHS: $\mathcal{B}_M = \bigcup_{\mu \in \mathcal{P}_M} \mathcal{H}_{\mu}$
and the norm $\norm{f}_{\mathcal{B}_M}$ is equal to the minimum of $\norm{f}_{\mathcal{H}_\mu}$ over all $\mu \in \mathcal{P}_M$~\citep{ma_barron_2022}.

To obtain the generalization guarantee, we utilize the Rademacher complexity.
The Rademacher complexity of a function class $\mathcal{F}$ of functions $f:\R^d \to \R^T$ is defined by
\begin{align*}
    \rcomp(\mathcal{F}) := \E_\sigma\qty[\sup_{f \in \mathcal{F}} \frac{1}{nT} \sum_{i=1}^n \sum_{t = 1}^T \sigma_{it}f_t(x_i)],
\end{align*}
where $\sigma_{it}$ is an i.i.d. Rademacher random variable ($P(\sigma_{it} = 1) = P(\sigma_{it} = -1) = 1/2$).
Then, we have the following bound for the mean-field neural networks.
\begin{lemma}\label{lem:rademacher-complexity}
    Assume that $h(x;w)$ satisfies Assumption~\ref{assumption:activation}.
    Define a class of the mean-field neural networks by
    \begin{align*}
        \mathcal{F}_{R, M} = \qty{f(x;a, \mu) \mid \kl(\nu \mid \mu) \leq M, \norm{a}_{L^2(\mu)}^2 \leq R}.
    \end{align*}
    Then, the Rademacher complexity of $\mathcal{F}_{R, M}$ is bounded by
    \begin{align*}
        \rcomp(\mathcal{F}_{R, M}) \leq \sqrt{\frac{R(4M+2T\log 2)}{nT}} = O\qty(\frac{R(M + T)}{nT}).
    \end{align*}
\end{lemma}
This is a generalization of the result in~\citet{chen_generalized_2020}.
See Appendix~\ref{sec:proof-rademacher-complexity} for the proof.

Then, we can derive the generalization error bound for the empirical risk minimizer $\hat \mu = \argmin_\mu \mathcal{G}_{\hat \rho}(\mu)$.
Note that the empirical risk minimizer $\hat \mu$ can be obtained by the MFLD as shown in Theorem~\ref{thm:convergence}.
\begin{theorem}\label{thm:generalization-bound}
    Assume that Assumption~\ref{assumption:activation} and~\ref{assumption:data} holds with $\sigma = 0$, $T=1$,
    and $f^\circ_1 \in \mathcal{B}_M, \norm{f^\circ_1}_{\mathcal{B}_M} \leq R$ for given $M, R > 0$.
    In addition, let $\lambda = 1/\sqrt{n}$, and $\lambda_a = 2M / R$.
    Then, with probability at least $1 - \delta$ over the choice of training examples, it holds that
    \begin{align*}
        \norm{f(\cdot;\hat \mu) - f^\circ}_{L^2(\rho_X)}^2 = O\qty((R + 1)\sqrt{\frac{M + 1 + \log 1/\delta}{n}}).
    \end{align*}
\end{theorem}
See Appendix~\ref{sec:proof-generalization-bound} for the proof.
Therefore, if $R = O(1)$, the mean-field neural networks can learn the Barron space with $n = O(M)$ samples.

Next, we show the lower bound of the estimation error for kernel methods.
For a given kernel $k$, a kernel method returns a function of the form $f(x) = \sum_{i=1}^n \alpha_i k(x, x_i)$
for $\alpha_i \in \R$. This type of estimator is called linear estimator.
The following theorem gives the lower bound of the sample complexity for any linear estimators.
\begin{theorem}\label{thm:lower-bound}
    Fix $m \in \N$ and
    let $d \geq \max\{2, m\}$ and $\rho_X$ be the uniform distribution on $\{-1, 1\}^d$ and $h(x;w) = \tanh(u \cdot x + b)$, where $w = (u, b) \in \R^{d + 1}$.
    In addition, let $H_n \subset L^2(\rho_X)$ be a set of functions of the form $\sum_{i=1}^n \alpha_i h_i(x)$
    and $d(f, H_n) = \inf_{\hat f \in H_n} \|f - \hat f\|_{L^2(\rho_X)}$.
    Then, there exist constants $c_1, c_2 > 0$ which is independent of $d$, such that, for every choice of fixed basis functions $h_1(x), \dots, h_n(x)$,
    it holds that
    \begin{align*}
        \sup_{f \in \mathcal{B}_M, \norm{f}_{\mathcal{B}_M}^2 \leq R} d(f, H_n) \geq \frac{1}{4}
    \end{align*}
    if $n \leq N / 2$ and $M = c_1 d \log d, R = c_2$ where $N = \binom{d}{m} = \Omega(d^m)$.
\end{theorem}
The proof can be found in Appendix~\ref{sec:proof-lower-bound}.
The key observation for the proof is that we can construct a function in the Barron space that approximates a single index model with certain regularity by taking a measure
which concentrates on a line toward a certain direction.
This theorem implies that any kernel estimator with $n = o(d^m)$ cannot learn the Barron space with $M = \Omega(d \log d)$.
This is in contrast to the mean-field neural networks which can learn the Barron space with $n = O(d \log d)$ samples as shown in Theorem~\ref{thm:generalization-bound}.
This is because the kernel methods cannot adapt the underlying RKHS under the target function.
Therefore, feature learning or kernel learning is essential to obtain good generalization results.

\section{Properties of the Kernel Induced by the First Layer}\label{sec:feature-learning}
In the previous section, we proved that feature learning is essential to obtain good generalization results
and two-layer neural networks trained by the MFLD can excel over kernel methods.
In this section, we study the properties of the kernel trained via the MFLD.
We show that in regression problem, the kernel induced by the first layer moves to increase kernel and parameter alignment.
We also proved that the presence of the noise $\varepsilon$ induces bias towards the large degrees of freedom.
To overcome this issue, we provide the label noise procedure, which provably converges to the global minima of the objective functional with the degrees of freedom regularization.

\subsection{Kernel and Parameter Alignment}
For simplicity, we consider the single output setting $T = 1$ and define $f^\circ = f^\circ_1$.
In addition, we consider $\tanh$ activation $h(x;w) = \tanh(u\cdot x + b)~(w = (u, b))$
and assume that $\rho_X = N(0, I)$.
To measure the adaptation of the kernel to the target function, we define the kernel alignment~\citet{cristianini_kernel-target_2001}. which is commonly used to measure the similarity between kernel and labels.
\begin{definition}
    For $\mu \in \mathcal{P}$, the empirical and population kernel alignment is defined by
    \begin{align*}
        \hat A(\mu) & = \frac{f^\circ(X)^\top \hat \Sigma_{\mu} f^\circ(X)}{\norm{f^\circ(X)}_2^2 \norm{\hat \Sigma_{\mu}}_{\mathrm{F}}},                                              \\
        A(\mu)      & = \frac{\E_{x \sim \rho_X, x'\sim \rho_X}[f^\circ(x) k_\mu(x, x')f^\circ(x')]}{\E_{\rho_X}[f^\circ(x)^2] \sqrt{\E_{x \sim \rho_X, x' \sim \rho_X}[k(x, x')^2]}}.
    \end{align*}
\end{definition}
Note that $\hat A(\mu)$ and $A(\mu)$ satisfy $0 \leq \hat A(\mu), A(\mu) \leq 1$
and larger $A(\mu)$ means that the kernel is aligned with the target function.

In this section, we consider the regression problem with squared loss.
if $\sigma = 0$, the limiting functional $U_{\hat \rho}(\mu)$ has the following explicit formula.
\begin{align*}
    U_{\hat \rho}(\mu) & = \frac{\bar \lambda_a}{2}  {f^\circ(X)}^\top (\hat \Sigma_{\mu} + n \bar \lambda_a I)^{-1} f^\circ(X).
\end{align*}
From Jensen's inequality, we have
\begin{align*}
    \hat A(\mu) \geq \frac{\bar \lambda_a \norm{f^\circ(X)}_2^2}{2U_{\hat \rho}(\mu)n} - \bar \lambda_a.
\end{align*}
See Lemma~\ref{lem:lower-bound-A} for the detailed derivation.
Therefore, the minimization of $U_{\hat \rho}(\mu)$ is equivalent to the maximization of the lower bound of the kernel alignment.

To derive a concrete expression of the kernel alignment, we assume that the target function $f^\circ$ is a single-index model
, which is a common structural assumption on the target function~\citep{bietti_learning_2022}.
\begin{assumption}\label{assumption:alignment}
    There exist $\tilde f:\R \to \R$, $u^\circ \in \R^{d}~(\norm{u^\circ}_2 = 1)$ such that
    $\tilde f$ is differentiable, $\|\tilde f'\|_\infty$, $\|\tilde f\|_{\infty} \leq 1$,
    $\E_{z \sim N(0, 1)}[\tilde f(z)] = 0$, and $f^\circ(x) = \tilde f(u^\circ \cdot x)$.
\end{assumption}

We also define the parameter alignment, which measures the similarity between the first layer parameters and the intrinsic direction of the target function.
\begin{definition}
    For $\mu \in \mathcal{P}$, the parameter alignment is defined as
    \begin{align*}
        P(\mu) = \E_{(u, b)\sim \mu}\qty[\frac{(u^\top u^\circ)^2}{\norm{u}^2}].
    \end{align*}
    Here, we define $\frac{(u^\top u^\circ)^2}{\norm{u}^2} = 0$ for $u = 0$.
\end{definition}
This is the expected cosine similarity between parameters and the target direction, and thus $0 \leq P(\mu) \leq 1$.
Note that larger $P(\mu)$ means that the first layer parameters are aligned with the intrinsic direction of the target function.

Then, we have the following result on the kernel for empirical risk minimizer $\hat \mu$.
\begin{theorem}\label{thm:alignment}
    Assume that Assumption~\ref{assumption:alignment} holds.
    Then, there exists universal constants $c_3, c_4, c_5$ satisfying the following:
    Let $\hat \mu$ be the minimizer of $\mathcal{G}_{\hat \rho}(\mu)$ with $n \geq c_3(d \log d + \log 1/ \delta)$,
    $\lambda = c_4 / (d \log d)$, and $\lambda_a = c_5 d \log d$ for $0 < \delta < 1$ and $d \geq 2$.
    Then, the kernel and parameter alignment for the initial distribution $\mu_0$
    and the empirical risk minimizer $\hat \mu$ satisfies
    \begin{align*}
        A(\mu_0) & = O(1/\sqrt{d}),~A(\hat \mu) = \Omega(1), \\
        P(\mu_0) & = O(1/d),~ P(\hat \mu) = \Omega(1),
    \end{align*}
    with probability at least $1 - \delta$ over the choice of samples.
\end{theorem}
See Appendix~\ref{sec:proof-alignment} for the proof.
In high-dimensional setting $d \gg 1$, $A(\hat \mu),P(\hat \mu) = \Omega(1)$ is a significant improvement over $P(\mu_0) = O(1/d),~ A(\mu_0) = O(1/\sqrt{d})$ at the initialization.
For the parameter alignment, similar results are shown in~\citet{mousavi-hosseini_neural_2022}, but they train only the first layer
and use the norm of the irrelevant directions as a measure of the alignment.
On the other hand, we consider the joint learning of the first and second layers and use the cosine similarity as a measure of the alignment.
In addition,~\citet{atanasov_neural_2021} studied the kernel alignment of NTK, but their analysis is limited to linear neural networks.
Furthermore,~\citet{ba2022high,wang2024nonlinear} studied the alignment of the conjugate kernel of two-layer neural networks after one-step gradient descent,
but their frameworks cannot deal with full training dynamics.

\subsection{Degrees of Freedom and Label Noise}\label{sec:label-noise}
To measure the complexity of the acquired kernel, we define the (empirical) degrees of freedom by
\begin{align*}
    d_{\lambda}(\mu) = \tr[\hat \Sigma_{\mu}(\hat \Sigma_{\mu} + n\lambda I)^{-1}]
\end{align*}
for $\lambda > 0$.
This quantity is the effective dimension of the kernel $k_\mu$ and
plays a crucial role in the analysis of kernel regression~\cite{caponnetto_optimal_2007}.
In addition, it is known that the degrees of freedom is related to the compressibility of neural networks~\cite{suzuki_spectral_2020}.

Under Assumption~\ref{assumption:data}, each label $Y_i$ can be decomposed as $Y_i = f_i^\circ(X) + \varepsilon_i$, where $\varepsilon_i$ is the observation noise.
Then, taking the expectation of $U_{\hat \rho}(\mu)$ with respect to $\varepsilon$
yields $\E_{\varepsilon}[U_{\hat \rho}(\mu)] = B - V + \text{const.}$, where $B = \frac{\bar \lambda_a}{2T} \sum_{i=1}^T \E_\varepsilon [{f_i^\circ(X)}^\top (\hat \Sigma_{\mu} + n \bar \lambda_a I)^{-1} f_i^\circ(X)]$ and $V = \frac{\bar \lambda_a \sigma^2}{6n}d_{\bar \lambda_a}(\mu)$.
See Lemma~\ref{lem:intrinsic-noise} for the derivation. Here, $B$ is the bias term, which corresponds to the alignment with the target function as shown in the previous section,
and $V$ is the variance term, which corresponds to the degrees of freedom.
Since $-V$ appears in $\E_{\varepsilon}[U_{\hat \rho}(\mu)]$, minimizing $U_{\hat \rho}(\mu)$ leads to the larger variance and the degrees of freedom.
We verify this phenomenon in Section~\ref{sec:numerical-experiments}.

To obtain good prediction performance, we need to minimize $B + V$ and control the bias-variance tradeoff.
Here, we consider the following objective functional with the degrees of freedom regularization:
\begin{align*}
    \mathcal{L}(\mu) & := \mathcal{G}_{\hat \rho}(\mu) + \frac{\bar \lambda_a \tilde \sigma^2}{6n} d_{\bar \lambda_a}(\mu).
\end{align*}
Here, $\tilde \sigma \geq 0$ controls the strength of the regularization.
Since this regularization is proportional to the variance $V$, minimizing $\mathcal{L}(\mu)$ would lead to smaller variance and better generalization.
To obtain the minimizer of the above functional $\mathcal{L}(\mu)$, we provide the label noise procedure, where we add independent \textit{label noise} to the training data
for implicit regularization.
In the discretized MFLD update, we add independent label noise $\tilde \varepsilon_i \sim \upsilon([-\tilde \sigma, \tilde \sigma]^{n})$ to $Y_i$ at each time step.
We use the noisy label $\tilde Y_i := Y_i + \tilde \varepsilon_i$ to train the second layer
and obtain $\tilde a^{(i)}_\mu := \argmin \frac{1}{nT} \sum_{i=1}^T \|\tilde Y_i - f(X;a, \mu)\|_2^2 + \frac{\bar \lambda_a}{2} \E_\mu[a(w)^2]$.
Here, the noisy limiting functional $G_{\tilde \varepsilon}(\mu)$ is defined as
\begin{align*}
    G_{\tilde \varepsilon}(\mu) & := \frac{1}{nT} \sum_{i=1}^T \norm{Y_i - f(X;\tilde a_\mu, \mu)}_2^2                                                    \\
                                & \quad + \frac{\bar \lambda_a}{2} \E_\mu\qty[\norm{\tilde a_\mu(w)}_2^2]+ \frac{\bar \lambda_w}{2} \E_\mu[\norm{w}_2^2].
\end{align*}
Note that we use the clean label $Y_i$ to define $G_{\tilde \varepsilon}(\mu)$ instead of $\tilde Y_i$.
Then, we update the first layer by the following discretized MFLD.
\begin{align*}
    w^{(k + 1)}_i = w^{(k)}_i - \eta \grad \fdv{G_{\tilde \varepsilon^{(k)}}(\mu_{W_k})}{\mu}\/(w^{(k)}_i) + \sqrt{2 \eta \lambda} \xi^{(k)}_i,
\end{align*}
where $\tilde \varepsilon^{(k)}$ is an independent noise at the $k$-th iteration.
In fact, the expectation of $G_{\tilde \varepsilon}(\mu) + \lambda \ent(\mu)$ with respect to $\tilde \varepsilon$ is equal to $\mathcal{L}(\mu)$
and the above procedure can be seen as the stochastic MFLD for minimizing $\mathcal{L}(\mu)$.
Indeed, the following theorem holds.
\begin{theorem}\label{thm:label-noise}
    Let $\mu^* = \argmin_\mu \mathcal{L}(\mu)$.
    Then, for $\eta < \min(1/4, 1/(4\alpha \lambda))$ and $0 \leq \tilde \sigma^2 / 3 \leq \lambda_{\min}\qty(\frac{1}{T} \sum_{i=1}^T Y_i {Y_i}^\top)$, we have
    \begin{align*}
         & \frac{1}{N} \E[\mathcal{L}^N(\mu^N_k)] - \mathcal{L}(\mu^*)                                                            \\
         & \quad \leq \exp(-\alpha \lambda \eta k / 2)(\E[\mathcal{L}^N(\mu^N_0)] - \mathcal{L}(\mu^*)) + \bar \delta_{\eta, N}',
    \end{align*}
    where $\bar \delta'_{\eta, N} = O(\eta + \frac{1}{N})$.
    Here, the expectation is taken with respect to the randomness of the label noise.
\end{theorem}
See Appendix~\ref{sec:proof-label-noise} for the proof.
Intuitively, the degrees of freedom represents a metric for quantifying the adaptability to noise
and the first layer performs the robust feature learning where the second layer is difficult to fit the label noise.
\citet{suzuki_optimal_2023} has shown the Bayes optimality of two-layer linear neural networks which minimizes the empirical risk with the degrees of freedom regularization.
However, they ignore the optimization aspect and directly assume that the optimal solution can be obtained.
Note that the condition on $\tilde \sigma^2$ is needed to ensure the convexity of the objective
and the multi-learning setting is necessary to set $\tilde \sigma > 0$ since $\frac{1}{T} \sum_{i=1}^T Y_i {Y_i}^\top$ must be full rank.
However, as shown in Section~\ref{sec:numerical-experiments},
the label noise procedure is effective even for the single output setting.

\section{Numerical Experiments}\label{sec:numerical-experiments}
To validate our theoretical results, we conduct numerical experiments with synthetic data.
Specifically, we consider $f^\circ(x) = x_1x_2$ for $d = 15$.
Then, the samples $\qty{(x^{(i)}, y^{(i)})}_{i=1}^n$ are independently generated so that
$x^{(i)}$ follows $N(0, I)$ and $y^{(i)} = f^\circ(x^{(i)}) + \varepsilon^{(i)}$, where $\varepsilon^{(i)} \sim \upsilon([\sigma, \sigma])$.
We consider a finite width neural network with the width $m = 2000$.
We trained the network via noisy gradient descent with $\eta = 0.2, \lambda = 0.004, \lambda_w = 0.25, \lambda_a = 0.25$ until $T = 10000$.
The results are averaged over 5 different random seeds.

First, we investigated the training dynamics of the kernel by changing the intrinsic noise $\sigma$.
As shown in Figure~\ref{fig:intrinsic-noise}, kernel moves to increase the kernel alignment and the degrees of freedom.
In addition, the intrinsic noise increases the degrees of freedom, which is consistent with our arguments in Section~\ref{sec:label-noise}.

Next, we demonstrated the effectiveness of the label noise procedure.
Fig.~\ref{fig:label-noise} shows the evolution of the degrees of freedom and the test loss during the training for different $\tilde \sigma$.
As expected, the label noise procedure reduces the degrees of freedom.
Moreover, the test loss is also improved,
which implies that the degrees of freedom is a good regularization for the generalization error.

\begin{figure}[h]
    \centering
    \includegraphics[width=\columnwidth]{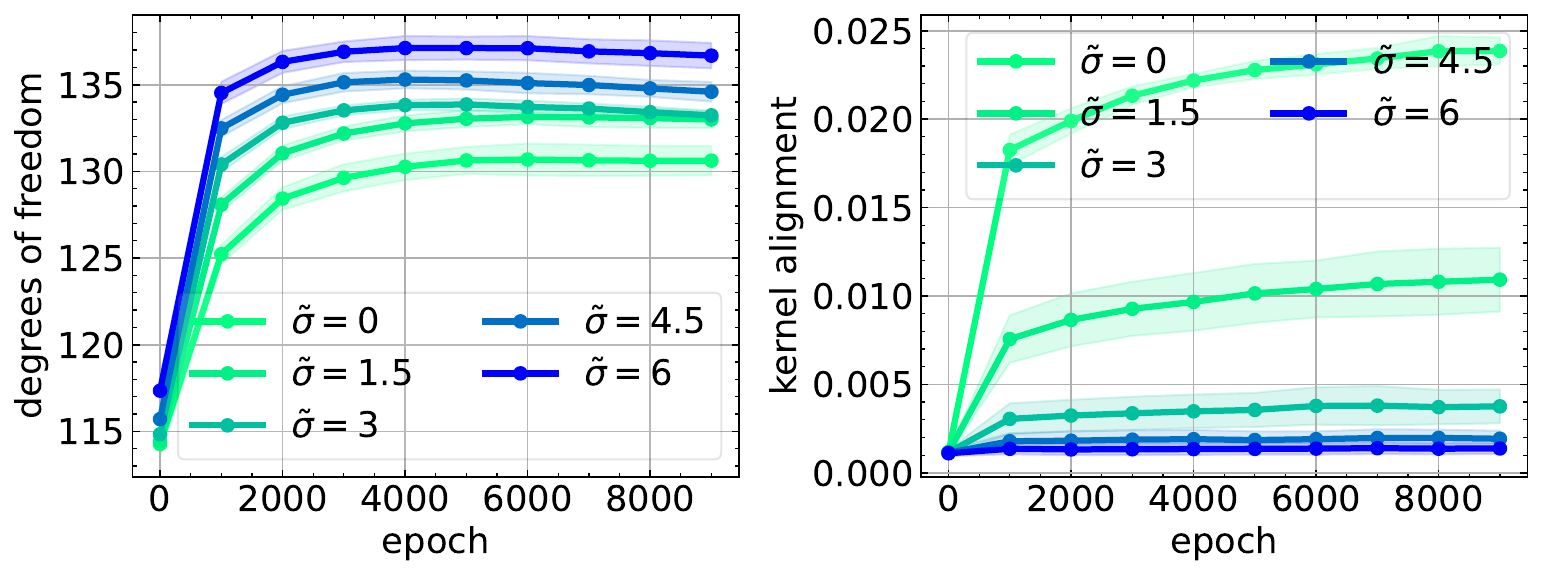}
    \vspace*{-0.7cm}
    \caption{Evolution of the kernel alignment and the degrees of freedom of neural network optimized by the MFLD}
    \label{fig:intrinsic-noise}
\end{figure}
\begin{figure}[h]
    \centering
    \includegraphics[width=\columnwidth]{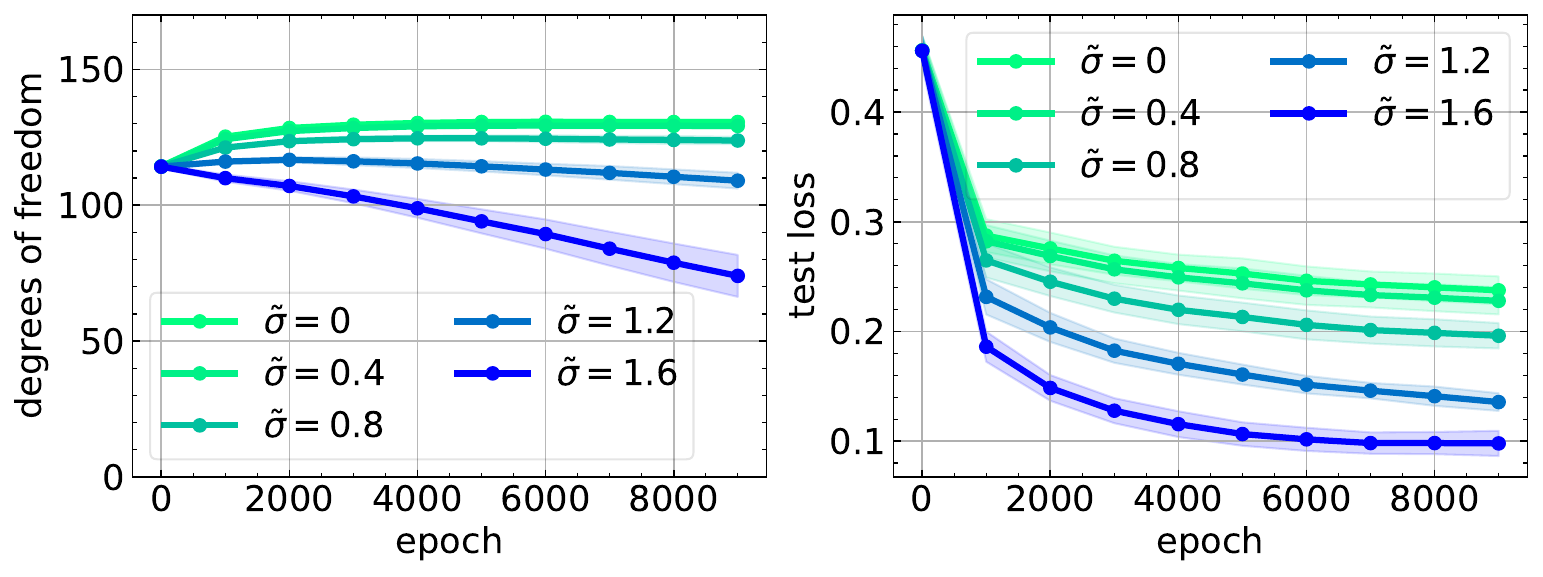}
    \vspace*{-0.7cm}
    \caption{Evolution of the degrees of freedom and the test error of the label noise procedure}
    \label{fig:label-noise}
\end{figure}
\section{Conclusion}
In this paper, we studied the feature learning ability of two-layer neural networks in the mean-field regime
via kernel learning formulation.
For that purpose, we proposed to use the two-timescale limit to analyze the training dynamics of the mean-field neural networks.
Then, we provided the linear convergence guarantee to the global optimum by showing the convexity of the limiting functional and derive the discretization error.
We also studied the generalization ability of the empirical risk minimizer
and proved that the feature learning is essential to obtain good generalization results for a union of multiple RKHSs.
Then, we showed that the kernel induced by the first layer moves to increase kernel and parameter alignment
and the intrinsic noise in labels induces bias towards the large degrees of freedom.
Finally, we proposed the label noise procedure to reduce the degrees of freedom and provided the global convergence guarantee.

\section*{Acknowledgements}
ST was partially supported by JST CREST (JPMJCR2015). TS was partially supported by JSPS KAKENHI (20H00576) and JST CREST (JPMJCR2115).


\bibliography{main}

\begin{thebibliography}{49}
\providecommand{\natexlab}[1]{#1}
\providecommand{\url}[1]{\texttt{#1}}
\expandafter\ifx\csname urlstyle\endcsname\relax
  \providecommand{\doi}[1]{doi: #1}\else
  \providecommand{\doi}{doi: \begingroup \urlstyle{rm}\Url}\fi

\bibitem[Abbe et~al.(2022)Abbe, Adsera, and Misiakiewicz]{abbe2022merged}
Abbe, E., Adsera, E.~B., and Misiakiewicz, T.
\newblock The merged-staircase property: a necessary and nearly sufficient condition for sgd learning of sparse functions on two-layer neural networks.
\newblock In \emph{Conference on Learning Theory}, pp.\  4782--4887. PMLR, 2022.

\bibitem[Arora et~al.(2019{\natexlab{a}})Arora, Du, Hu, Li, and Wang]{arora_fine-grained_2019}
Arora, S., Du, S., Hu, W., Li, Z., and Wang, R.
\newblock Fine-{Grained} {Analysis} of {Optimization} and {Generalization} for {Overparameterized} {Two}-{Layer} {Neural} {Networks}.
\newblock In \emph{Proceedings of the 36th {International} {Conference} on {Machine} {Learning}}, pp.\  322--332. PMLR, May 2019{\natexlab{a}}.
\newblock URL \url{https://proceedings.mlr.press/v97/arora19a.html}.
\newblock ISSN: 2640-3498.

\bibitem[Arora et~al.(2019{\natexlab{b}})Arora, Du, Hu, Li, Salakhutdinov, and Wang]{arora_exact_2019}
Arora, S., Du, S.~S., Hu, W., Li, Z., Salakhutdinov, R.~R., and Wang, R.
\newblock On {Exact} {Computation} with an {Infinitely} {Wide} {Neural} {Net}.
\newblock In \emph{Advances in {Neural} {Information} {Processing} {Systems}}, volume~32. Curran Associates, Inc., 2019{\natexlab{b}}.
\newblock URL \url{https://papers.neurips.cc/paper_files/paper/2019/hash/dbc4d84bfcfe2284ba11beffb853a8c4-Abstract.html}.

\bibitem[Atanasov et~al.(2021)Atanasov, Bordelon, and Pehlevan]{atanasov_neural_2021}
Atanasov, A., Bordelon, B., and Pehlevan, C.
\newblock Neural {Networks} as {Kernel} {Learners}: {The} {Silent} {Alignment} {Effect}.
\newblock In \emph{International {Conference} on {Learning} {Representations}}, October 2021.
\newblock URL \url{https://openreview.net/forum?id=1NvflqAdoom}.

\bibitem[Ba et~al.(2022)Ba, Erdogdu, Suzuki, Wang, Wu, and Yang]{ba2022high}
Ba, J., Erdogdu, M.~A., Suzuki, T., Wang, Z., Wu, D., and Yang, G.
\newblock High-dimensional asymptotics of feature learning: How one gradient step improves the representation.
\newblock \emph{Advances in Neural Information Processing Systems}, 35:\penalty0 37932--37946, 2022.

\bibitem[Bach(2017)]{bach_equivalence_2017}
Bach, F.
\newblock On the {Equivalence} between {Kernel} {Quadrature} {Rules} and {Random} {Feature} {Expansions}.
\newblock \emph{Journal of Machine Learning Research}, 18\penalty0 (21):\penalty0 1--38, 2017.
\newblock ISSN 1533-7928.
\newblock URL \url{http://jmlr.org/papers/v18/15-178.html}.

\bibitem[Bach et~al.(2004)Bach, Lanckriet, and Jordan]{bach_multiple_2004}
Bach, F.~R., Lanckriet, G. R.~G., and Jordan, M.~I.
\newblock Multiple kernel learning, conic duality, and the {SMO} algorithm.
\newblock In \emph{Twenty-first international conference on {Machine} learning - {ICML} '04}, pp.\ ~6, Banff, Alberta, Canada, 2004. ACM Press.
\newblock \doi{10.1145/1015330.1015424}.
\newblock URL \url{http://portal.acm.org/citation.cfm?doid=1015330.1015424}.

\bibitem[Bakry \& Émery(1985)Bakry and Émery]{bakry_diffusions_1985}
Bakry, D. and Émery, M.
\newblock Diffusions hypercontractives.
\newblock \emph{Séminaire de probabilités de Strasbourg}, 19:\penalty0 177--206, 1985.
\newblock URL \url{https://eudml.org/doc/113511}.
\newblock Publisher: Springer - Lecture Notes in Mathematics.

\bibitem[Baratin et~al.(2021)Baratin, George, Laurent, Hjelm, Lajoie, Vincent, and Lacoste-Julien]{baratin_implicit_2021}
Baratin, A., George, T., Laurent, C., Hjelm, R.~D., Lajoie, G., Vincent, P., and Lacoste-Julien, S.
\newblock Implicit {Regularization} via {Neural} {Feature} {Alignment}.
\newblock In \emph{Proceedings of {The} 24th {International} {Conference} on {Artificial} {Intelligence} and {Statistics}}, pp.\  2269--2277. PMLR, March 2021.
\newblock URL \url{https://proceedings.mlr.press/v130/baratin21a.html}.
\newblock ISSN: 2640-3498.

\bibitem[Barron(1993)]{barron_universal_1993}
Barron, A.
\newblock Universal approximation bounds for superpositions of a sigmoidal function.
\newblock \emph{IEEE Transactions on Information Theory}, 39\penalty0 (3):\penalty0 930--945, May 1993.
\newblock ISSN 0018-9448, 1557-9654.
\newblock \doi{10.1109/18.256500}.
\newblock URL \url{https://ieeexplore.ieee.org/document/256500/}.

\bibitem[Bartlett \& Mendelson(2001)Bartlett and Mendelson]{goos_rademacher_2001}
Bartlett, P.~L. and Mendelson, S.
\newblock Rademacher and {Gaussian} {Complexities}: {Risk} {Bounds} and {Structural} {Results}.
\newblock In Goos, G., Hartmanis, J., Van~Leeuwen, J., Helmbold, D., and Williamson, B. (eds.), \emph{Computational {Learning} {Theory}}, volume 2111, pp.\  224--240. Springer Berlin Heidelberg, Berlin, Heidelberg, 2001.
\newblock ISBN 978-3-540-42343-0 978-3-540-44581-4.
\newblock \doi{10.1007/3-540-44581-1_15}.
\newblock URL \url{http://link.springer.com/10.1007/3-540-44581-1_15}.
\newblock Series Title: Lecture Notes in Computer Science.

\bibitem[Bietti et~al.(2022)Bietti, Bruna, Sanford, and Song]{bietti_learning_2022}
Bietti, A., Bruna, J., Sanford, C., and Song, M.~J.
\newblock Learning single-index models with shallow neural networks.
\newblock In \emph{Advances in {Neural} {Information} {Processing} {Systems}}, May 2022.
\newblock URL \url{https://openreview.net/forum?id=wt7cd9m2cz2}.

\bibitem[Bietti et~al.(2023)Bietti, Bruna, and Pillaud-Vivien]{bietti_learning_2023}
Bietti, A., Bruna, J., and Pillaud-Vivien, L.
\newblock On {Learning} {Gaussian} {Multi}-index {Models} with {Gradient} {Flow}, November 2023.
\newblock URL \url{http://arxiv.org/abs/2310.19793}.
\newblock arXiv:2310.19793 [cs, math, stat].

\bibitem[Caponnetto \& De~Vito(2007)Caponnetto and De~Vito]{caponnetto_optimal_2007}
Caponnetto, A. and De~Vito, E.
\newblock Optimal {Rates} for the {Regularized} {Least}-{Squares} {Algorithm}.
\newblock \emph{Foundations of Computational Mathematics}, 7\penalty0 (3):\penalty0 331--368, July 2007.
\newblock ISSN 1615-3375, 1615-3383.
\newblock \doi{10.1007/s10208-006-0196-8}.
\newblock URL \url{http://link.springer.com/10.1007/s10208-006-0196-8}.

\bibitem[Chen et~al.(2023)Chen, Ren, and Wang]{chen_uniform--time_2023}
Chen, F., Ren, Z., and Wang, S.
\newblock Uniform-in-time propagation of chaos for mean field {Langevin} dynamics, November 2023.
\newblock URL \url{http://arxiv.org/abs/2212.03050}.
\newblock arXiv:2212.03050 [math, stat].

\bibitem[Chen et~al.(2020)Chen, Cao, Gu, and Zhang]{chen_generalized_2020}
Chen, Z., Cao, Y., Gu, Q., and Zhang, T.
\newblock A {Generalized} {Neural} {Tangent} {Kernel} {Analysis} for {Two}-layer {Neural} {Networks}.
\newblock In \emph{Advances in {Neural} {Information} {Processing} {Systems}}, volume~33, pp.\  13363--13373. Curran Associates, Inc., 2020.
\newblock URL \url{https://proceedings.neurips.cc/paper/2020/hash/9afe487de556e59e6db6c862adfe25a4-Abstract.html}.

\bibitem[Chizat(2022)]{chizat_mean-field_2022}
Chizat, L.
\newblock Mean-{Field} {Langevin} {Dynamics} : {Exponential} {Convergence} and {Annealing}.
\newblock \emph{Transactions on Machine Learning Research}, May 2022.
\newblock ISSN 2835-8856.
\newblock URL \url{https://openreview.net/forum?id=BDqzLH1gEm}.

\bibitem[Chizat \& Bach(2018)Chizat and Bach]{chizat_global_2018}
Chizat, L. and Bach, F.
\newblock On the global convergence of gradient descent for over-parameterized models using optimal transport.
\newblock \emph{Advances in neural information processing systems}, 31, 2018.
\newblock URL \url{https://proceedings-neurips-cc.utokyo.idm.oclc.org/paper_files/paper/2018/hash/a1afc58c6ca9540d057299ec3016d726-Abstract.html}.

\bibitem[Cristianini et~al.(2001)Cristianini, Shawe-Taylor, Elisseeff, and Kandola]{cristianini_kernel-target_2001}
Cristianini, N., Shawe-Taylor, J., Elisseeff, A., and Kandola, J.
\newblock On {Kernel}-{Target} {Alignment}.
\newblock In \emph{Advances in {Neural} {Information} {Processing} {Systems}}, volume~14. MIT Press, 2001.
\newblock URL \url{https://proceedings.neurips.cc/paper_files/paper/2001/hash/1f71e393b3809197ed66df836fe833e5-Abstract.html}.

\bibitem[Damian et~al.(2021)Damian, Ma, and Lee]{damian_label_2021}
Damian, A., Ma, T., and Lee, J.~D.
\newblock Label noise sgd provably prefers flat global minimizers.
\newblock \emph{Advances in Neural Information Processing Systems}, 34:\penalty0 27449--27461, 2021.
\newblock URL \url{https://proceedings-neurips-cc.utokyo.idm.oclc.org/paper/2021/hash/e6af401c28c1790eaef7d55c92ab6ab6-Abstract.html}.

\bibitem[Damian et~al.(2022)Damian, Lee, and Soltanolkotabi]{damian_neural_2022}
Damian, A., Lee, J., and Soltanolkotabi, M.
\newblock Neural {Networks} can {Learn} {Representations} with {Gradient} {Descent}.
\newblock In \emph{Proceedings of {Thirty} {Fifth} {Conference} on {Learning} {Theory}}, pp.\  5413--5452. PMLR, June 2022.
\newblock URL \url{https://proceedings.mlr.press/v178/damian22a.html}.
\newblock ISSN: 2640-3498.

\bibitem[E et~al.(2019)E, Ma, and Wu]{e_priori_2019}
E, W., Ma, C., and Wu, L.
\newblock \textit{{A} priori} estimates of the population risk for two-layer neural networks.
\newblock \emph{Communications in Mathematical Sciences}, 17\penalty0 (5):\penalty0 1407--1425, 2019.
\newblock ISSN 15396746, 19450796.
\newblock \doi{10.4310/CMS.2019.v17.n5.a11}.
\newblock URL \url{https://www.intlpress.com/site/pub/pages/journals/items/cms/content/vols/0017/0005/a011/}.

\bibitem[Fang et~al.(2019)Fang, Dong, and Zhang]{fang_over_2019}
Fang, C., Dong, H., and Zhang, T.
\newblock Over {Parameterized} {Two}-level {Neural} {Networks} {Can} {Learn} {Near} {Optimal} {Feature} {Representations}, October 2019.
\newblock URL \url{http://arxiv.org/abs/1910.11508}.
\newblock arXiv:1910.11508 [cs, math, stat].

\bibitem[Hayakawa \& Suzuki(2020)Hayakawa and Suzuki]{hayakawa_minimax_2020}
Hayakawa, S. and Suzuki, T.
\newblock On the minimax optimality and superiority of deep neural network learning over sparse parameter spaces.
\newblock \emph{Neural Networks}, 123:\penalty0 343--361, March 2020.
\newblock ISSN 0893-6080.
\newblock \doi{10.1016/j.neunet.2019.12.014}.
\newblock URL \url{https://www.sciencedirect.com/science/article/pii/S089360801930406X}.

\bibitem[Holley \& Stroock(1987)Holley and Stroock]{holley_logarithmic_1987}
Holley, R. and Stroock, D.
\newblock Logarithmic {Sobolev} inequalities and stochastic {Ising} models.
\newblock \emph{Journal of Statistical Physics}, 46\penalty0 (5):\penalty0 1159--1194, March 1987.
\newblock ISSN 1572-9613.
\newblock \doi{10.1007/BF01011161}.
\newblock URL \url{https://doi.org/10.1007/BF01011161}.

\bibitem[Hsu(2021)]{hsu2021dimension}
Hsu, D.
\newblock Dimension lower bounds for linear approaches to function approximation.
\newblock \emph{Daniel Hsu’s homepage}, 2021.

\bibitem[Hu et~al.(2019)Hu, Ren, Siska, and Szpruch]{hu2019mean}
Hu, K., Ren, Z., Siska, D., and Szpruch, L.
\newblock Mean-field langevin dynamics and energy landscape of neural networks.
\newblock \emph{arXiv preprint arXiv:1905.07769}, 2019.

\bibitem[Hu et~al.(2020)Hu, Ren, Siska, and Szpruch]{hu_mean-field_2020}
Hu, K., Ren, Z., Siska, D., and Szpruch, L.
\newblock Mean-{Field} {Langevin} {Dynamics} and {Energy} {Landscape} of {Neural} {Networks}, December 2020.
\newblock URL \url{http://arxiv.org/abs/1905.07769}.
\newblock arXiv:1905.07769 [math, stat].

\bibitem[Jacot et~al.(2018)Jacot, Gabriel, and Hongler]{jacot_neural_2018}
Jacot, A., Gabriel, F., and Hongler, C.
\newblock Neural {Tangent} {Kernel}: {Convergence} and {Generalization} in {Neural} {Networks}.
\newblock In \emph{Advances in {Neural} {Information} {Processing} {Systems}}, volume~31. Curran Associates, Inc., 2018.
\newblock URL \url{https://proceedings.neurips.cc/paper_files/paper/2018/hash/5a4be1fa34e62bb8a6ec6b91d2462f5a-Abstract.html}.

\bibitem[Li et~al.(2021)Li, Wang, and Arora]{li_what_2021}
Li, Z., Wang, T., and Arora, S.
\newblock What {Happens} after {SGD} {Reaches} {Zero} {Loss}? --{A} {Mathematical} {Framework}.
\newblock In \emph{International {Conference} on {Learning} {Representations}}, October 2021.
\newblock URL \url{https://openreview.net/forum?id=siCt4xZn5Ve}.

\bibitem[Ma \& Wu(2022)Ma and Wu]{ma_barron_2022}
Ma, C. and Wu, L.
\newblock The {Barron} space and the flow-induced function spaces for neural network models.
\newblock \emph{Constructive Approximation}, 55\penalty0 (1):\penalty0 369--406, 2022.
\newblock URL \url{https://link.springer.com/article/10.1007/s00365-021-09549-y}.
\newblock Publisher: Springer.

\bibitem[Mahankali et~al.(2023)Mahankali, HaoChen, Dong, Glasgow, and Ma]{mahankali_beyond_2023}
Mahankali, A.~V., HaoChen, J.~Z., Dong, K., Glasgow, M., and Ma, T.
\newblock Beyond {NTK} with {Vanilla} {Gradient} {Descent}: {A} {Mean}-{Field} {Analysis} of {Neural} {Networks} with {Polynomial} {Width}, {Samples}, and {Time}.
\newblock In \emph{Thirty-seventh {Conference} on {Neural} {Information} {Processing} {Systems}}, 2023.
\newblock URL \url{https://openreview-net.utokyo.idm.oclc.org/forum?id=Y2hnMZvVDm}.

\bibitem[Marion \& Berthier(2023)Marion and Berthier]{marion_leveraging_2023}
Marion, P. and Berthier, R.
\newblock Leveraging the two timescale regime to demonstrate convergence of neural networks, October 2023.
\newblock URL \url{http://arxiv.org/abs/2304.09576}.
\newblock arXiv:2304.09576 [cs, math, stat].

\bibitem[Maurer(2016)]{maurer_vector-contraction_2016}
Maurer, A.
\newblock A {Vector}-{Contraction} {Inequality} for {Rademacher} {Complexities}.
\newblock In Ortner, R., Simon, H.~U., and Zilles, S. (eds.), \emph{Algorithmic {Learning} {Theory}}, Lecture {Notes} in {Computer} {Science}, pp.\  3--17, Cham, 2016. Springer International Publishing.
\newblock ISBN 978-3-319-46379-7.
\newblock \doi{10.1007/978-3-319-46379-7_1}.

\bibitem[Mei et~al.(2018)Mei, Montanari, and Nguyen]{mei_mean_2018}
Mei, S., Montanari, A., and Nguyen, P.-M.
\newblock A mean field view of the landscape of two-layer neural networks.
\newblock \emph{Proceedings of the National Academy of Sciences}, 115\penalty0 (33), August 2018.
\newblock ISSN 0027-8424, 1091-6490.
\newblock \doi{10.1073/pnas.1806579115}.
\newblock URL \url{https://pnas.org/doi/full/10.1073/pnas.1806579115}.

\bibitem[Mousavi-Hosseini et~al.(2022)Mousavi-Hosseini, Park, Girotti, Mitliagkas, and Erdogdu]{mousavi-hosseini_neural_2022}
Mousavi-Hosseini, A., Park, S., Girotti, M., Mitliagkas, I., and Erdogdu, M.~A.
\newblock Neural {Networks} {Efficiently} {Learn} {Low}-{Dimensional} {Representations} with {SGD}.
\newblock In \emph{The {Eleventh} {International} {Conference} on {Learning} {Representations}}, 2022.
\newblock URL \url{https://openreview-net.utokyo.idm.oclc.org/forum?id=6taykzqcPD}.

\bibitem[Nitanda \& Suzuki(2017)Nitanda and Suzuki]{nitanda_stochastic_2017}
Nitanda, A. and Suzuki, T.
\newblock Stochastic {Particle} {Gradient} {Descent} for {Infinite} {Ensembles}, December 2017.
\newblock URL \url{http://arxiv.org/abs/1712.05438}.
\newblock arXiv:1712.05438 [cs, math, stat].

\bibitem[Nitanda \& Suzuki(2020)Nitanda and Suzuki]{nitanda_optimal_2020}
Nitanda, A. and Suzuki, T.
\newblock Optimal {Rates} for {Averaged} {Stochastic} {Gradient} {Descent} under {Neural} {Tangent} {Kernel} {Regime}.
\newblock In \emph{International {Conference} on {Learning} {Representations}}, 2020.
\newblock URL \url{https://openreview-net.utokyo.idm.oclc.org/forum?id=PULSD5qI2N1}.

\bibitem[Nitanda et~al.(2022)Nitanda, Wu, and Suzuki]{nitanda_convex_2022}
Nitanda, A., Wu, D., and Suzuki, T.
\newblock Convex {Analysis} of the {Mean} {Field} {Langevin} {Dynamics}.
\newblock In \emph{Proceedings of {The} 25th {International} {Conference} on {Artificial} {Intelligence} and {Statistics}}, pp.\  9741--9757. PMLR, May 2022.
\newblock URL \url{https://proceedings.mlr.press/v151/nitanda22a.html}.
\newblock ISSN: 2640-3498.

\bibitem[Suzuki \& Suzuki(2023)Suzuki and Suzuki]{suzuki_optimal_2023}
Suzuki, K. and Suzuki, T.
\newblock Optimal criterion for feature learning of two-layer linear neural network in high dimensional interpolation regime.
\newblock In \emph{The {Twelfth} {International} {Conference} on {Learning} {Representations}}, October 2023.
\newblock URL \url{https://openreview.net/forum?id=Jc0FssXh2R}.

\bibitem[Suzuki(2018)]{suzuki_fast_2018}
Suzuki, T.
\newblock Fast generalization error bound of deep learning from a kernel perspective.
\newblock In \emph{Proceedings of the {Twenty}-{First} {International} {Conference} on {Artificial} {Intelligence} and {Statistics}}, pp.\  1397--1406. PMLR, March 2018.
\newblock URL \url{https://proceedings.mlr.press/v84/suzuki18a.html}.
\newblock ISSN: 2640-3498.

\bibitem[Suzuki et~al.(2020)Suzuki, Abe, Murata, Horiuchi, Ito, Wachi, Hirai, Yukishima, and Nishimura]{suzuki_spectral_2020}
Suzuki, T., Abe, H., Murata, T., Horiuchi, S., Ito, K., Wachi, T., Hirai, S., Yukishima, M., and Nishimura, T.
\newblock Spectral {Pruning}: {Compressing} {Deep} {Neural} {Networks} via {Spectral} {Analysis} and its {Generalization} {Error}.
\newblock In \emph{Proceedings of the {Twenty}-{Ninth} {International} {Joint} {Conference} on {Artificial} {Intelligence}}, pp.\  2839--2846, Yokohama, Japan, July 2020. International Joint Conferences on Artificial Intelligence Organization.
\newblock ISBN 978-0-9992411-6-5.
\newblock \doi{10.24963/ijcai.2020/393}.
\newblock URL \url{https://www.ijcai.org/proceedings/2020/393}.

\bibitem[Suzuki et~al.(2022)Suzuki, Nitanda, and Wu]{suzuki_uniform--time_2022}
Suzuki, T., Nitanda, A., and Wu, D.
\newblock Uniform-in-time propagation of chaos for the mean-field gradient {Langevin} dynamics.
\newblock In \emph{The {Eleventh} {International} {Conference} on {Learning} {Representations}}, September 2022.
\newblock URL \url{https://openreview.net/forum?id=_JScUk9TBUn}.

\bibitem[Suzuki et~al.(2023{\natexlab{a}})Suzuki, Wu, and Nitanda]{suzuki_convergence_2023}
Suzuki, T., Wu, D., and Nitanda, A.
\newblock Convergence of mean-field {Langevin} dynamics: time-space discretization, stochastic gradient, and variance reduction.
\newblock In \emph{Thirty-seventh {Conference} on {Neural} {Information} {Processing} {Systems}}, November 2023{\natexlab{a}}.
\newblock URL \url{https://openreview.net/forum?id=9STYRIVx6u}.

\bibitem[Suzuki et~al.(2023{\natexlab{b}})Suzuki, Wu, Oko, and Nitanda]{suzuki_feature_2023}
Suzuki, T., Wu, D., Oko, K., and Nitanda, A.
\newblock Feature learning via mean-field {Langevin} dynamics: classifying sparse parities and beyond.
\newblock In \emph{Thirty-seventh {Conference} on {Neural} {Information} {Processing} {Systems}}, November 2023{\natexlab{b}}.
\newblock URL \url{https://https://openreview.net/forum?id=tj86aGVNb3}.

\bibitem[Vivien et~al.(2022)Vivien, Reygner, and Flammarion]{vivien_label_2022}
Vivien, L.~P., Reygner, J., and Flammarion, N.
\newblock Label noise (stochastic) gradient descent implicitly solves the {Lasso} for quadratic parametrisation.
\newblock In \emph{Proceedings of {Thirty} {Fifth} {Conference} on {Learning} {Theory}}, pp.\  2127--2159. PMLR, June 2022.
\newblock URL \url{https://proceedings.mlr.press/v178/vivien22a.html}.
\newblock ISSN: 2640-3498.

\bibitem[Wainwright(2019)]{wainwright2019high}
Wainwright, M.~J.
\newblock \emph{High-dimensional statistics: A non-asymptotic viewpoint}, volume~48.
\newblock Cambridge university press, 2019.

\bibitem[Wang et~al.(2024)Wang, Wu, and Fan]{wang2024nonlinear}
Wang, Z., Wu, D., and Fan, Z.
\newblock Nonlinear spiked covariance matrices and signal propagation in deep neural networks.
\newblock \emph{arXiv preprint arXiv:2402.10127}, 2024.

\bibitem[Yehudai \& Shamir(2019)Yehudai and Shamir]{yehudai_power_2019}
Yehudai, G. and Shamir, O.
\newblock On the {Power} and {Limitations} of {Random} {Features} for {Understanding} {Neural} {Networks}.
\newblock In \emph{Advances in {Neural} {Information} {Processing} {Systems}}, volume~32. Curran Associates, Inc., 2019.
\newblock URL \url{https://proceedings.neurips.cc/paper_files/paper/2019/hash/5481b2f34a74e427a2818014b8e103b0-Abstract.html}.

\end{thebibliography}
\bibliographystyle{icml2024}

\newpage
\appendix
\onecolumn

\section{Auxiliary Lemmas}
\begin{lemma}[\citet{holley_logarithmic_1987}]\label{lem:holley-stroock}
    Assume that a probability distribution $p(w)$ satisfies the LSI with a constant $\alpha > 0$.
    For a bounded perturbation $B(w):\R^{d'}\to \R$, define $p'(w) = p(w) \exp(B(w)) / \E_p[\exp(B(w))]$.
    Then, $p'$ satisfies the LSI with a constant $\frac{\alpha}{\exp(4\norm{B}_\infty)}$.
\end{lemma}

\begin{lemma}\label{lem:optimality-a}
    The optimal $i$-th second layer $a_\mu^{(i)}$ satisfies
    \begin{align*}
        a_\mu^{(i)}(w) = -\frac{1}{\bar \lambda_a} \E_\rho[\partial_1 l_i(f_i(x;a_\mu, \mu), y_i) h(x;w)].
    \end{align*}
    for any $w \in \R^{d'}$.
\end{lemma}
\begin{proof}
    From the optimality condition on $a_\mu^{(i)}$, it holds that
    \begin{align*}
        \pdv{L}{a_\mu^{(i)}}\/(a^{(i)}, \mu)(w) + \lambda \partial_{a_i} r(a_\mu(w), w)\mu(w) = \E[\partial_1 l_i(f_i(x;a_\mu^{(i)}, \mu), y_i) h(x;w)]\mu(w) + \bar\lambda_a a_\mu^{(i)}(w)\mu(w) = 0.
    \end{align*}
    Thus, we have
    \begin{align*}
        a_\mu^{(i)}(w) = -\frac{1}{\bar \lambda_a} \E_\rho[\partial_1 l_i(f_i(x;a_\mu, \mu), y) h(x;w)],
    \end{align*}
    which completes the proof.
\end{proof}
\begin{lemma}\label{lem:a-l2}
    Define $T:L^2(\mu) \to L^2(\rho_X)$ by
    \begin{align*}
        T(a) = \int a(w) h(x;w) \dd \mu(w)
    \end{align*}
    and its adjoint operator $T^*:L^2(\rho_X) \to L^2(\mu)$ by
    \begin{align*}
        T^*(f) = \int f(x) h(x;w) \dd \rho(x).
    \end{align*}
    For $l^2$-loss, the optimal $i$-th second layer $a_\mu^{(i)}$ has the following explicit formula.
    \begin{align*}
        a_\mu^{(i)}(w) & = (T^*T + \bar \lambda_a \id)^{-1} T^* f^\circ_i      \\
                       & = T^*(TT^* + \bar \lambda_a \id)^{-1} \bar f^\circ_i,
    \end{align*}
    where $f^\circ_i(x') := \E_\rho[y_i \mid x = x']$ is the conditional expectation of $y_i$ given $x$.
    In addition, if $\rho_X$ is the empirical distribution $\frac{1}{n}\sum_{i=1}^n \delta_{x_i}$,
    then, $a_\mu^{(i)}$ is written by
    \begin{align*}
        a_\mu^{(i)}(w) = h(X;w)^\top (\hat \Sigma_{\mu} + n \bar \lambda_a I)^{-1} Y_i.
    \end{align*}
\end{lemma}
\begin{proof}
    For $l^2$-loss, the optimality condition on $a$ is given by
    \begin{align*}
        \E_\rho[(T(a^{(i)}_\mu) - y)h(x;w)] + \bar \lambda_a a^{(i)}_\mu(w) = 0.
    \end{align*}
    Using $\E_\rho[y_ih(x;w)] = \E_{\rho_X}[f^\circ_i(x) h(x;w)] = T^* f^\circ_i$, we have
    \begin{align*}
        T^*T(a_\mu^{(i)}) + \bar \lambda_a a_\mu^{(i)} = T^* f^\circ_i.
    \end{align*}
    Since $\bar \lambda_a > 0$ and $(T^*T + \bar \lambda_a \id)$ is invertible,
    we arrive at
    \begin{align*}
        a^{(i)}_\mu = (T^*T + \bar \lambda_a \id)^{-1} T^* f^\circ_i.
    \end{align*}
    Since $(T^*T + \bar \lambda_a \id)T^* = T^*(TT^* + \bar \lambda_a \id)$,
    we have $T^*(TT^* + \bar \lambda_a \id)^{-1} = (T^*T + \bar \lambda_a \id)^{-1}T^*$,
    and thus $a_\mu^{(i)} = T^*(TT^* + \bar \lambda_a \id)^{-1} f^\circ_i$.
\end{proof}

\begin{lemma}\label{lem:a-boundedness-Lipschitzness}
    Assume that $l_i$ satisfies Assumption~\ref{assumption:convex-loss} or~\ref{assumption:l2-loss}.
    Then, for any $w \in \R^{d'}$, $a_\mu^{(i)}(w)$ satisfies the following conditions:
    \begin{itemize}
        \item $\abs{a_\mu^{(i)}(w)} \leq \frac{c_l}{\bar \lambda_1}=:B_a$
        \item $\norm{\grad_w a_\mu^{(i)}(w)}_2 \leq \frac{c_lc_R}{\bar \lambda_1}=:R_a$
        \item $\norm{\grad_w^2 a_\mu^{(i)}(w)}_{\ope} \leq \frac{c_lc_L}{\bar \lambda_1}=:L_a$
    \end{itemize}
\end{lemma}
\begin{proof}
    In the case of $\abs{\partial_1 l_i(z, y)} \leq c_l$, Lemma~\ref{lem:optimality-a} yields
    \begin{align*}
        \abs{a_\mu^{(i)}(w)} & = \frac{1}{\bar \lambda_a} \abs{\E_\rho[\partial_1 l_i(f_i(x;a_\mu, \mu), y) h(x;w)]} \\
                             & = \frac{c_l}{\bar \lambda_a}.
    \end{align*}
    In the case of $l(z, y) = \frac{1}{2}(z - y)^2$ and $\abs{y_i} \leq c_l$, Lemma~\ref{lem:a-l2} yields
    \begin{align*}
        \abs{a_\mu^{(i)}(w)} & = \abs{T^*(TT^* + \bar \lambda_a \id)^{-1} \bar y_i}                                       \\
                             & = \abs{\int h(x;w) [(TT^* + \bar \lambda_a \id)^{-1} \bar y_i](x)\dd \rho(x)}              \\
                             & \leq \norm{h(x;w)}_{L^2(\rho)}\norm{(TT^* + \bar \lambda_a \id)^{-1} \bar y_i}_{L^2(\rho)} \\
                             & \leq \norm{(TT^* + \bar \lambda_a \id)^{-1} \bar y_i}_{L^2(\rho)}.
    \end{align*}
    Here, the last inequality follows from $\abs{h(x;w)} \leq 1$.
    Since the operator norm of $(TT^* + \bar \lambda_a \id)^{-1}$ is bounded by $1 / \bar \lambda_a$, we have
    \begin{align*}
        \norm{(TT^* + \bar \lambda_a \id)^{-1} \bar y_i}_{L^2(\rho)} & \leq \frac{1}{\bar \lambda_a} \norm{\bar y_i}_{L^2(\rho)} \\
                                                                     & \leq \frac{c_l}{\bar \lambda_a}.
    \end{align*}
    Thus, we have the first assertion.

    In a similar way, we have
    \begin{align*}
        \norm{\grad_w a_\mu^{(i)}(w)}_2 & = \frac{1}{\bar \lambda_a} \norm{\E_\rho[\partial_1 l_i(f_i(x;a_\mu, \mu), y) \grad_w h(x;w)]} \\
                                        & \leq \frac{c_l}{\bar \lambda_a}\E[\norm{\grad_w h(x;w)}]                                       \\
                                        & \leq \frac{c_lc_R}{\bar \lambda_a},
    \end{align*}
    in the case of $\abs{\partial_1 l_i(z, y)} \leq c_l$.
    In addition, for a squared loss, we have
    \begin{align*}
        \norm{\grad a_\mu^{(i)}(w)}_{2} & = \norm{\int h(x;w) [(TT^* + \bar \lambda_a \id)^{-1} \bar y_i](x)\dd \rho(x)}_2                            \\
                                        & \leq \norm{\norm{\grad_w h(x;w)}_2}_{L^2(\rho)}\norm{(TT^* + \bar \lambda_a \id)^{-1} \bar y_i}_{L^2(\rho)} \\
                                        & \leq \frac{c_lc_R}{\bar \lambda_a}.
    \end{align*}
    Thus, we have the second assertion.

    Furthermore, we have
    \begin{align*}
        \norm{\grad_w^2 a_\mu^{(i)}(w)}_{\ope} & = \frac{1}{\bar \lambda_a} \norm{\E_\rho[\partial_1 l_i(f_i(x;a_\mu, \mu), y) \grad_w^2 h(x;w)]}_{\ope} \\
                                               & \leq \frac{c_l}{\bar \lambda_a}\E[\norm{\grad_w^2 h(x;w)}_{\ope}]                                       \\
                                               & \leq \frac{c_lc_L}{\bar \lambda_a},
    \end{align*}
    in the case of $\abs{\partial_1 l_i(z, y)} \leq c_l$.
    On the other hand, in the case of $l_i$ is a squared loss, we have
    \begin{align*}
        \norm{\grad_w^2 a_\mu^{(i)}(w)}_{\ope} & = \norm{\int \grad_w^2 h(x;w) [(TT^* + \bar \lambda_a \id)^{-1} \bar y_i](x)\dd \rho(x)}_{\ope}                    \\
                                               & \leq \norm{\norm{\grad_w^2 h(x;w)}_{\ope}}_{L^2(\rho)}\norm{(TT^* + \bar \lambda_a \id)^{-1} \bar y_i}_{L^2(\rho)} \\
                                               & \leq \frac{c_lc_L}{\bar \lambda_a}.
    \end{align*}
    This completes the proof.
\end{proof}
\begin{lemma}\label{lem:first-variation}
    Assume that each $l_i$ satisfies Assumption~\ref{assumption:convex-loss} or~\ref{assumption:l2-loss}.
    Then, we have
    \begin{align*}
        \abs{\fdv{U}{\mu}\/(\mu)(w)}                               & \leq \frac{\bar \lambda_a}{2} B_a^2   \\
        \norm{\grad_w \fdv{U}{\mu}\/(\mu)(w)}_2                    & \leq \bar \lambda_a R_aB_a,           \\
        \norm{\grad_w\grad_{w}^\top \fdv{U}{\mu}\/(\mu)(w)}_{\ope} & \leq \bar \lambda_a (R_a^2 + B_aL_a).
    \end{align*}
    for any $w \in \R^{d'}$.
\end{lemma}
\begin{proof}
    From Theorem~\ref{thm:convexity}, we have
    \begin{align*}
        \fdv{U(\mu)}{\mu}\/(w) = -\frac{\bar \lambda_a}{2T} \sum_{i=1}^T a^{(i)}_\mu(w)^2.
    \end{align*}
    Thus, Lemma~\ref{lem:a-boundedness-Lipschitzness} yields
    \begin{align*}
        \abs{\fdv{U}{\mu}\/(\mu)(w)}                                & \leq \frac{\bar \lambda_a}{2} B_a^2,   \\
        \norm{\grad_w \fdv{U}{\mu}\/(\mu)(w)}_2                     & \leq \bar \lambda_a B_a R_a,           \\
        \norm{\grad_w \grad_{w}^\top \fdv{U}{\mu}\/(\mu)(w)}_{\ope} & \leq \bar \lambda_a (R_a^2 + B_a L_a),
    \end{align*}
    which completes the proof.
\end{proof}
\begin{lemma}\label{lem:second-variation}
    Assume that each $l_i$ satisfies Assumption~\ref{assumption:convex-loss} or~\ref{assumption:l2-loss}.
    We have
    \begin{align*}
        \abs{\fdv[2]{U}{\mu}\/(\mu)(w, w')}                           & \leq B_a^2    \\
        \norm{\grad_w \fdv[2]{U}{\mu}\/(\mu)(w, w')}_2                & \leq 2R_aB_a, \\
        \norm{\grad_w\grad_{w'} \fdv[2]{U}{\mu}\/(\mu)(w, w')}_{\ope} & \leq 4R_a^2.
    \end{align*}
    for any $w, w' \in \R^{d'}$.
\end{lemma}
\begin{proof}
    Let $\bar l_i''(x') := \E_{\rho}[\partial_1^2 l_i(f(x), y_i) \mid x = x']$.
    Define $\Lambda, \Lambda^{1/2}: L^2(\rho) \to L^2(\rho)$ by
    \begin{align*}
        \Lambda_i(f)(x)       & = f(x) \bar l_i''(x),       \\
        \Lambda_i^{1/2}(f)(x) & = f(x) \bar l_i''(x)^{1/2},
    \end{align*}
    and $A^{(i)}(w) \in L^2(\rho)$ by
    \begin{align*}
        A^{(i)}(w)(x) = a^{(i)}_\mu(w)h(x;w)\bar l''(x)^{1/2}
    \end{align*}
    for a given $w \in \R^{d'}$.
    Note that $\Lambda, \Lambda^{1/2}. A$ are well-defined since $\bar l_i''(x) \geq 0$ from the convexity of $l_i$ with respect to the first argument.

    The second variation of $U(\mu)$ is given by
    \begin{align*}
        \fdv{}{\mu}\fdv{U(\mu)}{\mu}\qty(w, w') & = - \frac{\bar \lambda_a}{T} \sum_{i=1}^T a^{(i)}_\mu(w) \fdv{a^{(i)}_\mu(w)}{\mu}\/(w').
    \end{align*}
    Taking the first variation of the both sides of the optimality condition on $a^{(i)}_\mu(w)$ for a given $w$,
    we have
    \begin{align*}
         & \E_\rho[l_i''(f(x;a_\mu, \mu), y) \qty(a_\mu^{(i)}(w')h(x;w') + \int \fdv{a^{(i)}_\mu(w'')}{\mu}\/(w' ) h(x;w'') \dd \mu(w''))h(x;w)] + \bar \lambda_a \fdv{a_\mu^{(i)}(w)}{\mu}\/(w') \\
         & \quad = \qty[(T^* \Lambda_i T + \bar \lambda_a \id)\fdv{a_\mu^{(i)}(\cdot)}{\mu}\/(w')](w) + \qty[T^*\Lambda_i^{1/2}A^{(i)}(w')](w)                                                    \\
         & \quad = 0.
    \end{align*}
    Thus, we obtain
    \begin{align*}
        \fdv{}{\mu}\fdv{U(\mu)}{\mu} & = -\frac{\bar \lambda_a}{T} \sum_{i=1}^T a_\mu^{(i)}(w) \fdv{a_\mu^{(i)}(w)}{\mu}\/(w')                                                                                              \\
                                     & = -\frac{\bar \lambda_a}{T} \sum_{i=1}^T a_\mu^{(i)}(w)(T^*\Lambda_i T + \bar \lambda_a \id)^{-1}\qty[T^*\Lambda_i^{1/2}A^{(i)}(w')](w)                                              \\
                                     & = -\frac{\bar \lambda_a}{T} \sum_{i=1}^T a_\mu^{(i)}(w)\qty[T^*\Lambda_i^{1/2}(\Lambda^{1/2}TT^*\Lambda_i^{1/2} + \bar \lambda_a \id)^{-1}A^{(i)}(w')](w)                            \\
                                     & = -\frac{\bar \lambda_a}{T} \sum_{i=1}^T a_\mu^{(i)}(w) \int \qty[(\Lambda_i^{1/2}TT^*\Lambda_i^{1/2} + \bar \lambda_a \id)^{-1}A^{(i)}(w')](x) h(x;w) \bar l''(x)^{1/2} \dd \rho(x) \\
                                     & = -\frac{\bar \lambda_a}{T} \sum_{i=1}^T \int \qty[(\Lambda_i^{1/2}TT^*\Lambda_i^{1/2} + \bar \lambda_a \id)^{-1}A(w')](x) A(w)(x) \dd \rho(x)                                       \\
                                     & = -\frac{\bar \lambda_a}{T} \sum_{i=1}^T \langle A^{(i)}(w), (\Lambda_i^{1/2}TT^*\Lambda_i^{1/2} + \bar \lambda_a \id)^{-1}A^{(i)}(w')\rangle
    \end{align*}
    The second equality follows from the equality $A(A^*A + \id)^{-1} = (AA^* + \id)^{-1}A$ for any operator $A$ such that $(A^*A + \id), (AA^* + \id)$ are invertible.

    First, we bound $\norm{\norm{\grad_w A^{(i)}(w)}_2}_{L^2(\rho)}$ and $\norm{A^{(i)}(w)}_{L^2(\rho)}$.
    From Lemma~\ref{lem:a-boundedness-Lipschitzness}, we have
    \begin{align*}
        \norm{\norm{\grad_w A^{(i)}(w)}_2}_{L^2(\rho)} & = \norm{\norm{\grad_w a_\mu^{(i)}(w) h(x;w) + a_\mu^{(i)}(w) \grad_w h(x;w)}_2 \bar l''(x)^{1/2}}_{L^2(\rho)}                                              \\
                                                       & \leq \norm{\norm{\grad_w a_\mu^{(i)}(w)}_2 \bar l''(x)^{1/2}}_{L^2(\rho)} + \norm{\abs{a_\mu^{(i)}(w)}\norm{\grad h(x;w)}_2 \bar l''(x)^{1/2}}_{L^2(\rho)} \\
                                                       & \leq 2 R_a,                                                                                                                                                \\
        \norm{A^{(i)}(w)}_{L^2(\rho)}                  & = \norm{a_\mu^{(i)}(w)h(x;w)\bar l''(x)^{1/2}}_{L^2(\rho)}                                                                                                 \\
                                                       & \leq B_a
    \end{align*}
    since we have assumed that $0 \leq \partial_1^2 l_i(x, y) \leq 1$.
    The first assertion follows immediately from the above inequality and $\norm{(\Lambda_i^{1/2}TT^*\Lambda_i^{1/2} + \bar \lambda_a \id)^{-1}}_{\ope} \leq 1/\bar \lambda_a$.

    Then, $\norm{\grad_w\fdv[2]{U(\mu)}{\mu}\/(w, w')}$ is bounded as follows:
    \begin{align*}
        \norm{\grad_w\fdv[2]{U(\mu)}{\mu}\/(w, w')} & \leq \frac{\bar \lambda_a}{T} \sum_{i=1}^T \int \dd \rho(x) \norm{\grad_w A^{(i)}(w)(x)}_2\abs{[(\Lambda^{1/2}TT^*\Lambda^{1/2} + \bar \lambda_a \id)^{-1}A^{(i)}(w')](x)}                          \\
                                                    & \leq \frac{\bar \lambda_a}{T} \sum_{i=1}^T \norm{\norm{\grad_w A^{(i)}(w)(x)}_2}_{L^2(\rho)} \norm{(\Lambda^{1/2}TT^*\Lambda^{1/2} + \bar \lambda_a \id)^{-1}A^{(i)}(w')}_{L^2(\rho)}               \\
                                                    & \leq \frac{\bar \lambda_a}{T} \sum_{i=1}^T \norm{\norm{\grad_w A^{(i)}(w)(x)}_2}_{L^2(\rho)} \norm{(\Lambda^{1/2}TT^*\Lambda^{1/2} + \bar \lambda_a \id)^{-1}}_{\ope}\norm{A^{(i)}(w')}_{L^2(\rho)} \\
                                                    & \leq 2R_aB_a
    \end{align*}

    Finally, $\norm{\grad_w\grad_{w'}^\top \fdv[2]{U(\mu)}{\mu}\/(w, w')}_2$ is bounded as follows:
    \begin{align*}
        \norm{\grad_w\grad_{w'}^\top \fdv[2]{U(\mu)}{\mu}\/(w, w')}_{\ope} & \leq \frac{\bar \lambda_a}{T} \sum_{i=1}^T \tr[\int \int \dd \rho(x) \dd \rho(x') \grad_w A^{(i)}(w)(x) (\Lambda^{1/2}TT^*\Lambda^{1/2} + \bar \lambda_a \id)^{-1}(x, x') \grad_{w'}^\top A^{(i)}(w')(x')]          \\
                                                                           & \leq \frac{\bar \lambda_a}{T} \sum_{i=1}^T \int \int \dd \rho(x) \dd \rho(x') \norm{\grad_w A^{(i)}(w)(x)} (\Lambda^{1/2}TT^*\Lambda^{1/2} + \bar \lambda_a \id)^{-1}(x, x') \norm{\grad_{w'}^\top A^{(i)}(w')(x')} \\
                                                                           & \leq \frac{\bar \lambda_a}{T} \sum_{i=1}^T \norm{\norm{\grad_w A^{(i)}(w)(x)}}_{L^2(\rho)}^2 \norm{(\Lambda^{1/2}TT^*\Lambda^{1/2} + \bar \lambda_a \id)^{-1}}_{L^2(\rho)}                                          \\
                                                                           & \leq 4R_a^2.
    \end{align*}
\end{proof}

\section{Proofs for Section~\ref{sec:convergence-analysis}}
\subsection{Proof of Theorem~\ref{thm:convexity}}\label{sec:proof-convexity}
First, we derive the expression of the first variation of $G(\mu)$.
The envelope theorem implies that
\begin{align*}
    \fdv{G(\mu)}{\mu}\/(w) & = \pdv{F(a_\mu, \mu)}{\mu}\/(w) = \pdv{L(a_\mu, \mu)}{\mu}\/(w) + \lambda r(a_\mu(w), w).
\end{align*}
In addition, the first variations of $L$ w.r.t. $\mu$ and $a$ are given by
\begin{align*}
    \pdv{L(a_\mu, \mu)}{\mu}     & = \frac{1}{T} \sum_{i=1}^T \E_\rho[\partial_1 l_i(f_i(x;a,\mu), y)h(x;w)] a_\mu^{(i)}(w), \\
    \pdv{L(a_\mu, \mu)}{a^{(i)}} & = \frac{1}{T} \E_\rho[\partial_1 l_i(f_i(x;a,\mu), y)h(x;w)]\mu(w),
\end{align*}
respectively.
Therefore, we have
\begin{align*}
    \pdv{L(a_\mu, \mu)}{\mu}\/(w) = \sum_{i=1}^T \frac{a_\mu^{(i)}(w)}{\mu(w)} \pdv{L(a_\mu, \mu)}{a^{(i)}}\/(w).
\end{align*}
The first-order optimality condition on $a_\mu$ yields
\begin{align}
    \pdv{L(a_\mu, \mu)}{a^{(i)}}\/(w) = - \lambda \pdv{r(a_\mu(w), w)}{a^{(i)}} \mu(w), \label{eq:optimality-a}
\end{align}
which implies
\begin{align*}
    \pdv{L(a_\mu, \mu)}{\mu}\/(w) & = - \lambda \sum_{i=1}^T \pdv{r(a_\mu(w), w)}{a^{(i)}}a_\mu^{(i)}(w) \\
                                  & = - \lambda \langle \grad_a r(a_\mu(w), w), a_\mu(w)\rangle.
\end{align*}
Combining above arguments, we arrive at
\begin{align}
    \fdv{G(\mu)}{\mu}\/(w) & = -\lambda \langle \grad_a r(a_\mu(w), w), a_\mu(w)\rangle + \lambda r(a_\mu(w), w)  \label{eq:variation-G} \\
                           & = \lambda \qty(-\frac{\lambda_a}{2T} \norm{a_\mu(w)}_2^2 + \frac{\lambda_w}{2} \norm{w}_2^2).\notag
\end{align}

Next, we prove the convexity of $G(\mu)$.
From the convexity of $L(a, \mu)$ w.r.t. $a$, we have
\begin{align*}
    L(a_{\mu_1}, \mu_1) + \int \sum_{i=1}^T \pdv{L(a_{\mu_1}, \mu_1)}{a^{(i)}}\/(w)\qty(\frac{\mu_2(w)}{\mu_1(w)}a_{\mu_2}^{(i)}(w) - a_{\mu_1}^{(i)}(w)) \dd w \leq L\qty(\frac{\mu_2(w)}{\mu_1(w)}a_{\mu_2}, \mu_1) = L(a_{\mu_2}, \mu_2)
\end{align*}
for any $\mu_1, \mu_2 \in \mathcal{P}$.
Therefore, it holds that
\begin{align*}
     & G(\mu_1) + \int \sum_{i=1}^T \pdv{L(a_{\mu_1}, \mu_1)}{a^{(i)}}\/(w)\qty(\frac{\mu_2(w)}{\mu_1(w)}a_{\mu_2}^{(i)}(w) - a_{\mu_1}^{(i)}(w))\dd w \\
     & \quad +  \lambda r(a_{\mu_2}(w), w) \mu_2(w) -  \lambda r(a_{\mu_1}(w), w) \mu_1(w) \dd w                                                       \\
     & \quad \leq G(\mu_2).
\end{align*}
Thus, it is sufficient to show that
\begin{align*}
    \fdv{G(\mu_1)}{\mu}\/(w) (\mu_2(w) - \mu_1(w)) & \leq \sum_{i=1}^T \pdv{L(a_{\mu_1}, \mu_1)}{a^{(i)}}\/(w)\qty(\frac{\mu_2(w)}{\mu_1(w)}a_\mu^{(i)}(w) - a_{\mu_1}^{(i)}(w))\dd w \\
                                                   & \quad + \lambda r(a_{\mu_2}(w), w) \mu_2(w) -  \lambda r(a_{\mu_1}(w), w) \mu_1(w)
\end{align*}
for any $w$. To simplify the notation, we denote the LHS by $\rho_1(w)$ and the RHS by $\rho_2(w)$.
Substituting Eq.~\eqref{eq:variation-G} to $\rho_1(w)$, we have
\begin{align*}
    \rho_1(w) & = \lambda \qty[-\langle \grad_a r(a_{\mu_1}(w), w), a_{\mu_1}(w)\rangle + r(a_{\mu_1}(w), w)](\mu_2(w) - \mu_1(w)).
\end{align*}
On the other hand, substituting Eq.~\eqref{eq:optimality-a} to $\rho_2(w)$, we have
\begin{align*}
    \rho_2(w) & = - \lambda \langle \grad_a r(a_{\mu_1}(w), w) \qty(\mu_2(w)a_{\mu_2}(w) - a_{\mu_1}(w)\mu_1(w))\rangle + \lambda r(a_{\mu_2}(w), w) \mu_2(w) -  \lambda r(a_{\mu_1}(w), w) \mu_1(w).
\end{align*}
Therefore,
\begin{align*}
    \rho_2(w) - \rho_1(w) & = \lambda \mu_2(w) \qty[\langle -\grad_a {r(a_{\mu_1}(w), w)}(a_{\mu_2}(w) - a_{\mu_1}(w))\rangle + r(a_{\mu_2}(w), w) - r(a_{\mu_1}(w), w)] \\
                          & \geq 0.
\end{align*}
The last inequality follows from the convexity of $r(a, w)$ w.r.t. $a$.
This completes the proof.

\subsection{Proof of Lemma~\ref{lem:lsi}} \label{sec:proof-lsi}
Using Eq.~\eqref{eq:variation-G-l2}, we have
\begin{align*}
    p_\mu(w) \propto \exp\qty(-\frac{1}{\lambda} \fdv{G}{\mu}(w)) = \exp\qty(\frac{\lambda_a}{2T} \norm{a(w)}^2_2 - \frac{\lambda_w}{2} \norm{w}_2^2).
\end{align*}
The distribution $p_0(w) \propto \exp(-\frac{\lambda_w}{2} \norm{w}_2^2)$ satisfies the LSI with constant $\alpha = \lambda_w$ since $\frac{\lambda_w}{2}\norm{w}_2^2$ is $\lambda_w$-strongly convex~\citep{bakry_diffusions_1985}.
Thus, Lemma~\ref{lem:holley-stroock} implies that $p_\mu$ satisfies the LSI with constant $\alpha = \frac{\lambda_w}{\exp(4\norm{B}_\infty)}$, where $B(w) = \frac{\lambda_a}{2T} \norm{a_\mu(w)}_2^2$.
From Lemma~\ref{lem:a-boundedness-Lipschitzness}, we have $\abs{a_\mu^{(i)}(w)} \leq \frac{c_l}{\bar \lambda_a}$ for any $w$.
Therefore, $\norm{B(w)}_\infty \leq \frac{c_l^2\lambda_a}{2\bar \lambda_a^2}$.
This completes the proof.

\subsection{Proof of Theorem~\ref{thm:convergence}}\label{sec:proof-convergence}
The continuous time result follows from Lemma~\ref{lem:lsi}, Theorem~\ref{thm:convexity}, and the result in~\citet{nitanda_convex_2022}.

For the discretized time result, we follow the framework in~\citet{suzuki_convergence_2023}.
First, we prove the following lemma.
\begin{lemma}\label{lem:discretization-error}
    For any $w \in \R^{d'}$, $U(\mu)$ satisfies the following conditions:
    \begin{itemize}
        \item $\norm{\grad \fdv{U}{\mu}\/(\mu)(w) - \grad \fdv{U}{\mu}\/(\mu')(w')} \leq L_U (W_2(\mu, \mu') + \norm{w - w'}), \abs{\fdv[2]{U(\mu)}{\mu}\/(w, w')} \leq L_U$ for $L_U = 4R_a^2 + \bar \lambda_a(B_aL_a + R_a^2) + B_a^2$.
        \item $\norm{\grad \fdv{U}{\mu}\/(\mu)(w)} \leq R_U$ for $R_U = \bar \lambda_a B_a R_a$.
    \end{itemize}
\end{lemma}
\begin{proof}
    From Lemma~\ref{lem:first-variation} and~\ref{lem:second-variation}, we have
    \begin{align*}
        \norm{\grad_w^2\fdv{U(\mu)}{\mu}\/(w)}_{\ope} & \leq \bar \lambda_a (B_a L_a + R_a^2), \\
        \abs{\fdv[2]{U}{\mu}\/(w, w')}                & \leq B_a^2
    \end{align*}
    for any $w, w' \in \R^{d'}$.
    Thus, $\fdv{U(\mu)}{\mu}\/(w)$ is $\bar \lambda_a (B_a L_a + R_a^2)$-smooth and it holds that
    \begin{align*}
        \norm{\grad \fdv{U}{\mu}\/(\mu)(w) - \grad \fdv{U}{\mu}\/(\mu)(w')}_2 & \leq \bar \lambda_a (B_a L_a + R_a^2) \norm{w - w'}_2.
    \end{align*}

    Let $\mu_t = t\mu + (1 - t)\mu'$. Then, we have
    \begin{align*}
        \norm{\grad \fdv{U}{\mu}\/(\mu)(w) - \grad \fdv{U}{\mu}\/(\mu')(w)}_2 & \leq \int_0^1 \norm{\grad_w \dv{}{t} \fdv{U}{\mu}\/(\mu_t)(w)}_2 \dd t                            \\
                                                                              & = \int_0^1 \norm{\int \grad_w \fdv[2]{U}{\mu}\/(\mu_t)(w, w')(\mu'(w') - \mu(w')) \dd w'}_2 \dd t \\
    \end{align*}
    From the definition of $W_1$, for any $\varepsilon > 0$, there exists a coupling $\pi$ of $\mu$ and $\mu'$ such that
    \begin{align*}
        \int \norm{w - w'}_2 \dd \pi(w, w') \leq W_1(\mu, \mu') + \varepsilon \leq W_2(\mu, \mu') + \varepsilon.
    \end{align*}
    Thus, we have
    \begin{align*}
        \norm{\int \grad_w \fdv[2]{U}{\mu}\/(\mu_t)(w, w')(\mu'(w') - \mu(w')) \dd w'}_2
         & = \norm{\int \qty(\grad_w \fdv[2]{U}{\mu}\/(\mu_t)(w, w'') - \grad_w \fdv[2]{U}{\mu}\/(\mu_t)(w, w')) \dd \pi(w, w')}_2 \\
         & \leq \int \norm{\grad_w \fdv[2]{U}{\mu}\/(\mu_t)(w, w'') - \grad_w \fdv[2]{U}{\mu}\/(\mu_t)(w, w')}_2 \dd \pi(w, w')    \\
         & \leq \int 4R_a^2 \norm{w'' - w'}_2 \dd \pi(w, w')                                                                       \\
         & \leq 4R_a^2 W_2(\mu, \mu') + 4R_a^2 \varepsilon.
    \end{align*}
    The second inequality follows from the Lipschitz continuity of $\grad_w \fdv[2]{U}{\mu}\/(\mu_t)(w, w')$ from Lemma~\ref{lem:second-variation}.
    Since $\varepsilon$ is arbitrary, we have
    \begin{align*}
        \norm{\grad \fdv{U}{\mu}\/(\mu)(w) - \grad \fdv{U}{\mu}\/(\mu')(w)} & \leq \int_0^1 \norm{\int \grad_w \fdv[2]{U}{\mu}\/(\mu_t)(w, w')(\mu'(w') - \mu(w')) \dd w'} \dd t \\
                                                                            & \leq 4R_a^2 W_2(\mu, \mu').
    \end{align*}
    This completes the proof.
\end{proof}

Combining the above lemma and Theorem 3 in~\citet{suzuki_convergence_2023}, we have for any $\lambda \alpha \eta \leq 1/2$ and $\eta \leq 1/4$,
\begin{align*}
    \frac{1}{N} \E[\mathcal{G}^N(\mu_k^{(N)})] - \mathcal{G}(\mu^*) \leq \exp(-\lambda \alpha \eta k) \qty(\frac{1}{N} \E[\mathcal{G}^N(\mu_k^{(N)})] - \mathcal{G}(\mu^*)) + \frac{2}{\lambda \alpha}\bar L^2C_1(\lambda \eta + \eta^2) + \frac{2C_\lambda}{\lambda \alpha N},
\end{align*}
where $\bar R^2 = \E\qty[\norm{w_i^{(0)}}_2] + \frac{1}{\bar \lambda_w}\qty[\qty(\frac{1}{4}+\frac{1}{\bar \lambda_w})R_U^2 + \lambda d']$,
$\bar L = L_U + \bar \lambda_w$, $C_1 = 8[R_U^2 + \bar \lambda_w \bar R^2 + d']$, and $C_\lambda = 2\lambda L_U \alpha + 2\lambda^2 L_U^2\bar R^2$.

\subsection{Proof of Proposition~\ref{prop:convergence-function}}\label{sec:proof-convergence-function}
As shown in Lemma 3 in~\citet{suzuki_convergence_2023}, we have
\begin{align*}
    W_2^2(\mu_k^{(N)}, {\mu^*}^N) & \leq \frac{2}{\lambda \alpha} (\mathcal{G}^N(\mu_k^{(N)}) - N\mathcal{G}(\mu^*)).
\end{align*}
Let $\gamma$ be a coupling of $\mu_k^{(N)}$ and ${\mu^*}^N$.
Then, for $(W, W^*) \sim \gamma$, we have
\begin{align*}
    (k(x, y)_{\mu_W} - k(x, y)_{\mu^*})^2 & \leq 2(k_{\mu_W}(x, y) - k_{\mu_{W^*}}(x, y))^2 + 2(k_{\mu_{W^*}}(x, y) - k_{\mu^*}(x, y))^2.
\end{align*}
Let $k(x, y;w) := h(x;w)h(y;w)$.
This is $2c_R$-Lipschitz continuous with respect to $w$ for any $(x, y) \in S \times S$.
Then, for the first term in the right hand side, we have
\begin{align*}
    \qty(k_{\mu_W}(x, y) - k_{\mu_{W^*}}(x, y))^2
     & \leq \qty(\frac1N \sum_{i=1}^N k(x, y;w_i) - k(x, y;w^*_i))^2 \\
     & \leq \frac1N \sum_{i=1}^N \qty(k(x, y;w_i) - k(x, y;w^*_i))^2 \\
     & \leq \frac1N \sum_{i=1}^N 4c_R^2 \norm{w_i - w^*_i}_2^2       \\
     & \leq \frac1N 4c_R^2 \norm{W - W^*}_{\mathrm{F}}^2
\end{align*}
for any $(x, y) \in S\times S$.
For the second term, we have
\begin{align*}
    P\qty(\sup_{(x, y) \in S \times S}\abs{k_{\mu_{W^*}}(x, y) - k_{\mu^*}(x, y)} \geq 2\E_{\sigma, W^*}\qty[\sup_{(x, y) \in S \times S}\abs{\frac1N \sum_{i=1}^N \sigma_i k(x, y;w_i)} + \sqrt{\frac{2t}{N}}]) \leq \exp(-t)
\end{align*}
for any $t > 0$ by the same argument in Lemma 2 in \citet{suzuki_feature_2023}.
From the relation between Gaussian complexity and Rademacher complexity, and contraction inequality in~\citet{goos_rademacher_2001}, we have
\begin{align*}
    \E\qty[\sup_{(x, y) \in S \times S}\abs{\frac1N \sum_{i=1}^N \sigma_i k(x, y;w_i)}] & \leq c c_R \E\qty[\sup_{(x, y) \in S \times S}\abs{\frac1N \sum_{i=1}^N \sum_{j=1}^{d'} \varepsilon_{ij} w_{ij}}] \\
                                                                                        & \leq \E\qty[\frac{cc_R}{N}\sqrt{\sum_{i=1}^N \norm{w_i}_{2}^2}]                                                   \\
                                                                                        & \leq \frac{cc_R}{\sqrt{N}}\sqrt{\E_{\mu}[\norm{w}_2^2]},
\end{align*}
where $\varepsilon_{ij}$ is a Gaussian random variable with mean 0 and variance 1, and $c$ is a universal constant.
From the optimality of $\mu^*$, we have $\frac{\bar \lambda_w}{2} \E_{\mu^*}[\norm{w}_2^2] \leq \mathcal{G}(\mu^*) \leq \mathcal{G}(\mu_0)$.
Thus, we have
\begin{align*}
    \E\qty[\sup_{(x, y) \in S \times S}\abs{\frac1N \sum_{i=1}^N \sigma_i k(x, y;w_i)}] & \leq \frac{cc_R}{\sqrt{N}} \sqrt{\frac{2\mathcal{G}(\mu_0)}{\bar \lambda_w}}
\end{align*}
and
\begin{align*}
    P\qty[\qty(\sup_{(x, x') \in S \times S}\abs{k_{\mu_{W^*}}(x, y) - k_{\mu^*}(x, y)})^2 - \frac{4c^2\mathcal{G}(\mu_0)c_R^2}{\bar \lambda_w N} \geq \frac{4t}{N}] \leq \exp(-t)
\end{align*}
Thus, it holds that
\begin{align*}
     & \E\qty[\qty(\sup_{(x, x') \in S \times S}\abs{k_{\mu_{W^*}}(x, y) - k_{\mu^*}(x, y)})^2]- \frac{4c^2\mathcal{G}(\mu_0)c_R^2}{\bar \lambda_w N}                                             \\
     & \quad \leq \int_0^\infty P\qty[\qty(\sup_{(x, x') \in S \times S}\abs{k_{\mu_{W^*}}(x, y) - k_{\mu^*}(x, y)})^2 - \frac{4c^2\mathcal{G}(\mu_0)c_R^2}{\bar \lambda_w N} \geq \tau] \dd \tau \\
     & \quad \leq\int_0^\infty \exp(-N\tau / 4) \dd t = \frac{4}{N}.
\end{align*}

Combining above arguments, we arrive at
\begin{align*}
    \E_{\gamma}\qty[\sup_{(x, y) \in S \times S}(k(x, y)_{\mu_W} - k(x, y)_{\mu^*})^2]
     & \leq \frac8N c_R^2\E_{\gamma}[\norm{W - W^*}_{\mathrm{F}}^2] + \frac{4}{N} + \frac{4c^2\mathcal{G}(\mu_0)c_R^2}{\bar \lambda_w N} .
\end{align*}
By taking the infimum of the coupling $\gamma$, we have
\begin{align*}
    \E_{W \sim \mu_k^{(N)}}\qty[\sup_{(x, y) \in S \times S}(k(x, y)_{\mu_W} - k(x, y)_{\mu^*})^2] & \leq \frac2N c_R^2 W_2^2(\mu_k^{(N)}, {\mu^*}^N) + \frac{4}{N} + \frac{4c^2\mathcal{G}(\mu_0)c_R^2}{\bar \lambda_w N} =O\qty(\frac{\Delta}{N}).
\end{align*}
For any $f \in L^2(\rho_X)$, we have
\begin{align*}
    \abs{(\Sigma_{\mu_W} - \Sigma_{\mu^*})(f)(x)}
     & = \int (k_{\mu_{W}}(x, x') - k_{\mu^*}(x, x')) f(x') \dd \rho_X(x')                               \\
     & \leq \sqrt{\int (k_{\mu_{W}}(x, x') - k_{\mu^*}(x, x'))^2 \dd \rho_X(x')} \norm{f}_{L^2(\rho_X)},
\end{align*}
which yields
\begin{align*}
    \norm{(\Sigma_{\mu_W} - \Sigma_{\mu^*})(f)}_{L^2(\rho_X)}^2 & \leq \norm{f}_{L^2(\rho_X)}^2 \int (k_{\mu_{W}}(x, x') - k_{\mu^*}(x, x'))^2 \dd \rho_X(x') \dd \rho_X(x).
\end{align*}
This implies $\norm{\Sigma_{\mu_W} - \Sigma_{\mu^*}}_{\ope}^2 \leq \E_{x, x'}[(k_{\mu_{W}}(x, x') - k_{\mu^*}(x, x'))^2] \leq \Delta / N$.
Since $f_i(x;\mu)$ is the optimal solution of $\min_{f \in \mathcal{H}_\mu} \E_\rho[l_i(f(x), y)] + \frac{\bar \lambda_a}{2} \norm{f}_{\mathcal{H}_{\mu}}$,  where $l_i$ is the squared loss,
$f_i(x;\mu) = \Sigma_\mu (\Sigma_\mu + \bar \lambda_a \id)^{-1} \bar y = \int k(x, x')\alpha(x')\dd\rho_X(x')$,
where $\alpha_\mu(x) = (\Sigma_\mu + \bar \lambda_a \id)^{-1} \bar y$.
From the identity $A^{-1} - A'^{-1} = -A^{-1}(A - A')A'^{-1}$ for any invertible operator $A, A'$,
we have
\begin{align*}
    \norm{\alpha_{\mu_W} - \alpha_{\mu^*}}_{\rho_X}^2 & = \norm{(\Sigma_{\mu_W} + \bar \lambda_a \id)^{-1} (\Sigma_{\mu_W} - \Sigma_{\mu^*})(\Sigma_{\mu^*} + \bar \lambda_a \id)^{-1} \bar y}_{L^2(\rho_X)}^2 \\
                                                      & \leq \frac{c_0^2\norm{\Sigma_{\mu_W} - \Sigma_{\mu^*}}_{\ope}^2}{\bar \lambda_a^4}.
\end{align*}
Thus, for any $x \in S$, we have
\begin{align*}
    (f_i(x;\mu_W) - f_i(x;\mu^*))^2
     & \leq 2\qty(\int k_{\mu_W'}(x, x')\alpha_{\mu_W}(x') - k_{\mu^*}(x, x') \alpha_{\mu_W}(x') \dd \rho_X(x'))^2                                                                               \\
     & \quad + 2\qty(\int k_{\mu^*}(x, x')\alpha_{\mu_W}(x') - k_{\mu^*}(x, x') \alpha_{\mu^*}(x') \dd \rho_X(x'))^2                                                                             \\
     & \leq 2 \sup_{x, x' \in S \times S}(k_{\mu_W}(x, x') - k_{\mu^*}(x, x'))^2 \norm{\alpha_{\mu_W}}_{L^2(\rho_X)}^2                                                                           \\
     & \quad + 2 \norm{\alpha_{\mu_W} - \alpha_{\mu^*}}_{\rho_X}^2                                                                                                                               \\
     & \leq \frac{2c_0^2\sup_{x, x' \in S \times S}(k_{\mu_W}(x, x') - k_{\mu^*}(x, x'))^2}{\bar \lambda_a^2} + \frac{2c_0^2 \norm{\Sigma_{\mu_W} - \Sigma_{\mu^*}}_{\ope}^2}{\bar \lambda_a^4}.
\end{align*}
By taking the supremum over $x$ and expectation over $W$, we obtain the result.
\section{Proofs for Section~\ref{sec:generalization-error}}

\subsection{Lemmas for Section~\ref{sec:generalization-error}}
\begin{lemma}\label{lem:approximation-tanh}
    For a given $\delta, \tau > 0$, we have
    \begin{align*}
        \abs{\tanh(\tau z) - (1[z \geq 0] - 1[z < 0])} & \leq 2e^{-2\tau\abs{z}}
    \end{align*}
    for any $z \in \R$.
\end{lemma}
\begin{proof}
    From the definition of $\tanh$, we have, for any $z \geq 0$,
    \begin{align*}
        \abs{\tanh(\tau z) - (1[z \geq 0] - 1[z < 0])}
         & = 1 - \frac{e^{\tau z} - e^{-\tau z}}{e^{\tau z} + e^{-\tau z}} \\
         & = \frac{2e^{-\tau z}}{e^{\tau z} + e^{-\tau z}}                 \\
         & \leq 2e^{-2\tau z}.
    \end{align*}
    Similarly, for $z < 0$, we have
    \begin{align*}
        \abs{\tanh(\tau z) - (1[z \geq 0] - 1[z < 0])}
         & = 1 + \frac{e^{\tau z} - e^{-\tau z}}{e^{\tau z} + e^{-\tau z}} \\
         & = \frac{2e^{\tau z}}{e^{\tau z} + e^{-\tau z}}                  \\
         & \leq 2e^{2\tau z}.
    \end{align*}
    Combining above arguments, we obtain the result.
\end{proof}

\begin{lemma}\label{lem:dim-lower-bound}
    Let $\rho_X$ be the uniform distribution on $[0, 1]^d$ and
    $\mathcal{S} := \qty{\sin(2\pi w\cdot x) \mid w \in \qty{0, 1}^d, \norm{w}_1 = k}$ be a subset of $L^2(\rho_X)$.
    Furthermore, for any fixed basis functions $\qty{h_j}_{j=1}^n \subset{L^2(\rho_X)}$,
    let $H_n$ be a span of $\qty{h_j}_{j=1}^n$.
    Then, for any $\varepsilon \in [0, 1]$, we have
    \begin{align*}
        \sup_{\psi \in \mathcal{S}} d(\psi, H_n) & \geq \varepsilon
    \end{align*}
    if $n \leq N(1 - \varepsilon)$, where $N = \abs{\mathcal{S}} = \binom{d}{k}$.
\end{lemma}
\begin{proof}
    Assume that $d(\psi, H_n) < \varepsilon$ for any $\psi \in \mathcal{S}$.
    From Theorem 1 in~\citet{hsu2021dimension}, we have $n > N (1 - \varepsilon)$ since $\mathcal{S}$ is an orthonormal system in $L^2(\rho_X)$ and $\abs{\mathcal{S}} = N$.
    This contradicts $n \leq N(1 - \varepsilon)$, which completes the proof.
\end{proof}
\begin{lemma}\label{lem:approximation-sin}
    For $\varepsilon, r, r_x > 0$, let $\lambda_w = 1$, $h(x;w) = \tanh(x \cdot u + b)~(w = (u, b))$.
    Then, for any $u^\circ \in \R^d$ and $\tilde f:\R \to \R$ which is $1$-Lipschitz continuous and differentiable almost everywhere,
    there exists $\mu, a$ such that $\kl(\nu \mid \mu) = O\qty(\frac{r^2}{\varepsilon^2} + d \log \frac{drr_x}{\varepsilon}), \norm{a}_{L^2(\mu)} = r$ and
    \begin{align*}
        \abs{\tilde f(w^\circ \cdot x) - \qty[f(x;a, \mu) + \frac{1}{2}(\tilde f(r) + \tilde f(-r))]} \leq \varepsilon
    \end{align*}
    for any $x \in \R^d$ such that $\abs{u^\circ \cdot x} \leq r$ and $\norm{x} \leq r_x \sqrt{d}$.
\end{lemma}
\begin{proof}
    Let $a(u, b) = r \tilde a(b / \tau)$ for $\tilde a(b):\R \to \R, \norm{\tilde a}_\infty \leq 1$
    and $\mu(w) = \mu(u, b) := \mu(u) \mu(b)$, where $\mu(u) = N(\tau u^\circ, \sigma^2 I), \mu(b) = \upsilon([-\tau r, \tau r])$ for $\tau, \sigma > 0$.
    In addition, let $\bar g(x) = \E_{b\sim \mu_\tau}[r \tilde a(b / \tau) \tanh(\tau x \cdot w^\circ + b)]$.
    Then, we have
    \begin{align*}
        \abs{\bar g(x) - f(x;a, \mu)} & \leq \int \abs{r \tilde a(b / \tau)}\abs{\tanh(x \cdot u + b) - \tanh(\tau x \cdot u^\circ + b)} \dd \mu(\tilde u, b) \\
                                      & \leq \int \abs{r \tilde a(b / \tau)}\abs{x \cdot u - \tau x u^\circ} \dd \mu(\tilde u, b)                             \\
                                      & \leq r \sqrt{\int \abs{x \cdot (u - \tau u^\circ)}^2 \dd \mu(u)}                                                      \\
                                      & \leq r \sqrt{\int \norm{x}^2 \norm{u - \tau u^\circ}^2 \dd \mu(u)}                                                    \\
                                      & \leq r r_x \sqrt{d}\sigma.
    \end{align*}

    Let $\tilde g(x;a) := \int_{-k}^0 \frac{1}{2} \tilde a(b') (1[u^\circ \cdot x + b' \geq 0] - 1[u^\circ \cdot x + b' < 0]) \dd b'$.
    Since
    \begin{align*}
        \bar g(x;a) & = \int r\tilde a(b / \tau) \tanh(\tau x \cdot u^\circ + b) \dd \mu(b)                             \\
                    & = \int_{-\tau r}^{\tau r} \frac{1}{2\tau}\tilde a(b / \tau) \tanh(\tau x \cdot u^\circ + b) \dd b \\
                    & = \int_{-r}^{r} \frac{1}{2}\tilde a(b') \tanh(\tau (x \cdot u^\circ + b')) \dd b',
    \end{align*}
    it holds that
    \begin{align*}
        \abs{\bar g(x) - \tilde g(x)} & \leq \int_{-r}^r \frac{1}{2} \abs{\tilde a(b)}\abs{\tanh(\tau(x \cdot u^\circ + b'))- (1[x \cdot u^\circ + b' \geq 0] - 1[x \cdot u^\circ + b' < 0])} \dd b' \\
                                      & \leq \int_{-\infty}^\infty \frac{1}{2} \abs{\tanh(\tau(x \cdot u^\circ + b'))- (1[x \cdot u^\circ + b' \geq 0] - 1[x \cdot u^\circ + b' < 0])} \dd b'        \\
                                      & \leq \int_0^\infty e^{-2\tau z} \dd z                                                                                                                        \\
                                      & = 1 / (2\tau)
    \end{align*}
    where we used Lemma~\ref{lem:approximation-tanh} for the last inequality.
    Since $-r \leq u^\circ \cdot x \leq r$, we have
    \begin{align*}
        \tilde g(x) & = \frac{1}{2}\qty[\int_{-r}^r \tilde a(b')  1[u^\circ \cdot x + b' \geq 0] \dd b' - \int_{-r}^r \tilde a(b')  1[u^\circ \cdot x + b' < 0] \dd b'] \\
                    & =  \frac{1}{2}\qty[\int_{-u^\circ \cdot x}^r \tilde a(b') \dd b' - \int_{-r}^{-u^\circ \cdot x} \tilde a(b') \dd b'].
    \end{align*}
    By letting
    \begin{align*}
        \tilde a(b) = \begin{cases}
                          \tilde f'(-b) & \text{if } b \in [-r, r] \\
                          0             & \text{otherwise}
                      \end{cases},
    \end{align*}
    we obtain
    \begin{align*}
        \tilde g(x;\tilde a)
         & = \frac{1}{2}\qty[\int_{-w^\circ \cdot x}^r \tilde f'(-b') \dd b' - \int_{-r}^{-w^\circ \cdot x} \tilde f'(-b') \dd b'] \\
         & = \frac{1}{2}\qty[\int_{-r}^{w^\circ \cdot x} \tilde f'(b') \dd b' - \int_{w^\circ \cdot x}^r \tilde f'(b') \dd b']     \\
         & = \tilde f(w^\circ \cdot x) - \frac{1}{2}\qty[\tilde f(r) + \tilde f(-r)].                                              \\
    \end{align*}

    Combining above results, we have
    \begin{align*}
        \abs{\tilde f(w^\circ \cdot x) - f(x;a, \mu)} & \leq \abs{\tilde f(w^\circ \cdot x) - \bar g(x)} + \abs{\bar g(x) - \tilde g(x)} \\
                                                      & \leq \frac{r r_x \sqrt{d} \sigma}{2} + \frac{1}{2\tau}                           \\
                                                      & \leq \varepsilon
    \end{align*}
    by letting $\tau = 1 / \varepsilon, \sigma = \varepsilon / (rr_x \sqrt{d})$.

    Finally, we show that $f(x;a, \mu)$ is in $\mathcal{B}_M$ and $\norm{f}_{\mathcal{B}_M} \leq R$.
    Since $u, b$ is independent each other when $(u, b) \sim \mu$, we have
    \begin{align*}
        \kl(\nu \mid \mu) & = \kl\qty(N\qty(0, I) \mid N(\tau u^\circ, \sigma^2 I)) + \kl\qty(N\qty(0, 1) \mid u([-r\tau, r\tau])).
    \end{align*}
    For the first term, we have
    \begin{align*}
        \kl\qty(N\qty(0, I) \mid N(\tau u^\circ, \sigma^2 I)) & = \frac{1}{2}\qty[d \log \frac{1}{\sigma^2} - d + \norm{\tau u^\circ}^2 + d \sigma^2]                                      \\
                                                              & \leq \frac{1}{2} \qty[d \log \frac{r^2r_x^2d}{\varepsilon} - d + \frac{r^2}{\varepsilon^2} + \frac{\varepsilon}{r^2r_x^2}] \\
                                                              & = O\qty(d \log \frac{drr_x}{\varepsilon} + \frac{r^2}{\varepsilon^2})
    \end{align*}
    For the second term, we have
    \begin{align*}
        \kl\qty(N\qty(0, 1) \mid u([-r\tau, r\tau])) & \leq \int_{-r\tau}^{r\tau} \log \frac{1/(2r\tau)}{\frac{1}{\sqrt{2\pi}} e^{-b^2/2}} 1/(2r\tau) \dd b \\
                                                     & = \frac{(r\tau)^2}{6} + \frac{1}{2}\log(2\pi) - \log(2r\tau)                                         \\
                                                     & = O\qty(\frac{r^2}{\varepsilon^2})
    \end{align*}
    Thus, it follows that
    \begin{align*}
        \kl(\nu \mid \mu) & = O\qty(\frac{r^2}{\varepsilon^2} + d \log \frac{drr_x}{\varepsilon})
    \end{align*}
    In addition, $\norm{a}_\infty \leq r$ yields $\norm{a}_{L^2(\mu)} \leq r$.
    This completes the proof.
\end{proof}
\subsection{Proof of Lemma~\ref{lem:rademacher-complexity}}\label{sec:proof-rademacher-complexity}

Since Rademacher complexity is smaller than Gaussian complexity~\citep{wainwright2019high}, it suffices to bound the Gaussian complexity
$\gcomp(\mathcal{F}):=\E_{\varepsilon_{it} \sim N(0, 1)}\qty[\sup_{f \in \mathcal{F}} \frac{1}{nT} \sum_{i=1}^n \sum_{t = 1}^T \varepsilon_{it}f_t(x_i)]$.
Let $Z_t(w) := \frac{1}{\sqrt{n}}\sum_{i=1}^n \varepsilon_{it} h(x_i;w)$.
Note that $Z_t(w)$ follows a Gaussian distribution with mean $0$
and variance $\sigma(w)^2 := \frac{1}{n}\sum_{i=1}^n h(x_i;w)^2 \leq 1$ independently.
Then, we have
\begin{align}
    \gcomp(\mathcal{F}_{M, R}) & := \E_{\varepsilon}\qty[\sup_{f\in \mathcal{F}_{R, M}} \frac{1}{nT} \sum_{t=1}^T \sum_{i=1}^n \varepsilon_{it} f(x_i)]       \notag                                                                                                                                                           \\
                               & = \E_{\varepsilon}\qty[\sup_{\kl(\nu \mid \mu) \leq M}\sup_{\frac{1}{T}\sum_{t=1}^T \norm{a^{(t)}}_{L^2(\mu)}^2 \leq R} \frac{1}{nT} \sum_{i=1}^T \sum_{i=1}^n \varepsilon_{it} \int a^{(t)}(w) h(x_i;w) \dd \mu(w)] \notag                                                                   \\
                               & = \frac{1}{\sqrt{n}}\E_{\varepsilon}\qty[\sup_{\kl(\nu \mid \mu) \leq M} \sup_{\frac{1}{T}\sum_{t=1}^T \norm{a^{(t)}}_{L^2(\mu)}^2 \leq R} \frac{1}{T}\sum_{t=1}^T \int a^{(t)}(w) Z_t(w) \dd \mu(w)]                 \notag                                                                  \\
                               & \leq \frac{1}{\sqrt{n}}\E_{\varepsilon}\qty[\sup_{\kl(\nu \mid \mu) \leq M} \sup_{\frac{1}{T}\sum_{t=1}^T \norm{a^{(t)}}_{L^2(\mu)}^2 \leq R} \frac{1}{T}\sum_{t=1}^T \norm{a^{(t)}(w)}_{L^2(\mu)} \norm{Z_t(w)}_{L^2(\mu)}]                 \notag                                           \\
                               & \leq \frac{1}{\sqrt{n}}\E_{\varepsilon}\qty[\sup_{\kl(\nu \mid \mu) \leq M} \sup_{\frac{1}{T}\sum_{t=1}^T \norm{a^{(t)}}_{L^2(\mu)}^2 \leq R}  \sqrt{\frac{1}{T}\sum_{t=1}^T \norm{a^{(t)}(w)}_{L^2(\mu)}^2} \sqrt{\frac{1}{T}\sum_{t=1}^T\norm{Z_t(w)}_{L^2(\mu)}^2}]                 \notag \\
                               & = \sqrt{\frac{R}{n}} \E_{\varepsilon}\qty[\sup_{\kl(\nu \mid \mu) \leq M} \sqrt{\frac{1}{T}\sum_{t=1}^T\norm{Z_t(w)}_{L^2(\mu)}^2}]                                              \notag                                                                                                       \\
                               & \leq \sqrt{\frac{R}{n}} \sqrt{ \E_{\varepsilon}\qty[\sup_{\kl(\nu \mid \mu) \leq M} \int \frac{1}{T} \sum_{t=1}^T Z_t(w)^2 \dd \mu(w)]}   \label{eq:gcomp1}
\end{align}
For the last inequality, we used the fact that $\sqrt{\cdot}$ is monotonically increasing and Jensen's inequality.
From the Donsker-Varadhan duality formula of the KL-divergence, we have
\begin{align}
    \frac{1}{\gamma} \E_{\varepsilon}\qty[\sup_{\kl(\nu \mid \mu) \leq M} \gamma \int \frac{1}{T} \sum_{t=1}^T Z_t(w)^2 \dd \mu(w)]
     & \leq \frac{1}{\gamma} \qty{M + \E_{\varepsilon}\qty[\log \int \exp(\frac{\gamma}{T} \sum_{t=1}^T Z_t(w)^2)\dd \nu(w)]}  \notag            \\
     & \leq \frac{1}{\gamma} \qty{M + \log \int \E_{\varepsilon}\qty[\exp(\frac{\gamma}{T} \sum_{t=1}^T Z_t(w)^2)] \dd \nu(w)} \notag            \\
     & \leq \frac{1}{\gamma} \qty{M + \log \int \E_{\varepsilon}\prod_{t=1}^T\qty[\exp(\frac{\gamma}{T} Z_t(w)^2)] \dd \nu(w)} \label{eq:gcomp2}
\end{align}
for any $\gamma > 0$.
Since $Z_t(w) \sim N(0, \sigma(w)^2)$, we have
\begin{align*}
    \E_{Z \sim N(0, \sigma(w)^2)}\qty[\exp(\frac{\gamma}{T} Z^2)] & = \frac{1}{\sqrt{2\pi \sigma(w)^2}} \int e^{\frac{\gamma}{T} Z^2}e^{-\frac{Z^2}{2\sigma(w)^2}} \dd Z   \\
                                                                  & = \frac{1}{\sqrt{2\pi \sigma(w)^2}} \int e^{-\qty[\frac{1}{2\sigma(w)^2} - \frac{\gamma}{T}]Z^2} \dd Z \\
                                                                  & = \frac{1}{\sqrt{2\pi \sigma(w)^2}} \sqrt{\frac{\pi}{\frac{1}{2\sigma(w)^2} - \frac{\gamma}{T}}}       \\
                                                                  & = \sqrt{\frac{1}{1 - 2\frac{\gamma}{T} \sigma(w)^2}}.
\end{align*}
By letting $\gamma = \frac{T}{4}$, we have
\begin{align}
    \E_{Z}\qty[\exp(\gamma Z^2)] & = \sqrt{\frac{1}{1 - \sigma(w)^2 / 2}} \leq \sqrt{2}. \label{eq:gcomp3}
\end{align}
since $\sigma(w)^2 \leq 1$.
Combining Eq.~\eqref{eq:gcomp1},~\eqref{eq:gcomp2}, and~\eqref{eq:gcomp3}, we have
\begin{align*}
    \gcomp(\mathcal{F}_{R, M})
     & \leq \sqrt{\frac{R}{n}} \sqrt{ \E_{\varepsilon_i}\qty[\sup_{\kl(\nu \mid \mu) \leq M} \int \frac{1}{T} \sum_{t=1}^T Z_t(w)^2 \dd \mu(w)]} \\
     & \leq \sqrt{\frac{R}{n}} \sqrt{ \frac{1}{\gamma} \qty{M + \log \int \prod_{t=1}^T \E_{\varepsilon}\qty[\exp(\gamma Z_t(w)^2)] \dd \nu(w)}} \\
     & \leq \sqrt{\frac{R}{n}} \sqrt{ \frac{1}{\gamma} \qty{M + \log \int \sqrt{2}^T \dd \nu(w)}}                                                \\
     & \leq \sqrt{\frac{R}{n}} \sqrt{4 \qty{M/T + \log \sqrt{2}}}.                                                                               \\
\end{align*}
This completes the proof.

\subsection{Proof of Theorem~\ref{thm:generalization-bound}}\label{sec:proof-generalization-bound}
From the definition of $\mathcal{B}_M$, there exists $\mu^\circ \in \mathcal{P}_M$ and $a^\circ \in L^2(\mu)(\norm{a^\circ}_{L^2(\mu)}^2 \leq R)$
such that $f^\circ(x) = f(x;a^\circ, \mu^\circ)$.
Let $\hat a = a_{\hat \mu}$.
Then, from the optimality of $\hat \mu$ and $\hat a$, we have
\begin{align*}
    \mathcal{G}(\hat \mu) & = L(\hat a, \hat \mu) + \lambda \qty(\frac{\lambda_a}{2}\norm{\hat a}_{L^2(\hat \mu)}^2 + \kl(\nu \mid \hat\mu))             \\
                          & \leq \mathcal{G}(\mu^\circ)                                                                                                  \\
                          & = L(a^\circ, \mu^\circ) + \lambda \qty(\frac{\lambda_a}{2}\norm{a_{\mu^\circ}}_{L^2(\mu^\circ)}^2 + \kl(\nu \mid \mu^\circ)) \\
                          & \leq 2\lambda M
\end{align*}
since we set $\lambda_a = 2 M / R$.
Thus, it holds that
\begin{align*}
    \norm{a_{\hat \mu}}_{L^2(\hat \mu)}^2 & \leq 2R, \\
    \kl(\nu \mid \hat \mu)                & \leq 2M.
\end{align*}
From Lemma~\ref{lem:rademacher-complexity}, Rademacher complexity of $\mathcal{F}_{2R,2M}$ is bounded as follows:
\begin{align*}
    \rcomp(\mathcal{F}_{2R,2M}) & = O\qty(\sqrt{\frac{R(M+1)}{n}}).
\end{align*}

For any $f = \int a(w)h(x;w)\dd \mu(w) \in \mathcal{F}_{2R, 2M}$, we have
\begin{align*}
    \norm{f}_\infty & \leq \int \abs{a(w)} \dd \mu(w) \\
                    & \leq \norm{a}_{L^2(\mu)}        \\
                    & \leq \sqrt{2R}.
\end{align*}
Thus, for any $f \in \mathcal{F}_{2R, 2M}$, $l(f(x), y) = l(f(x), f^\circ(x)) \leq 4R$ and $\abs{l'(f(x), y)} \leq 2\sqrt{2R}$.
Let $\mathcal{F}' = \qty{(x, y) \mapsto l(f(x), y) \mid f \in \mathcal{F}_{2R, 2M}}$.
Utilizing the standard uniform bound~\cite{wainwright2019high}, for any $\delta \in [0, 1]$, we have
\begin{align*}
    \sup_{g \in \mathcal{F}'}\qty{\E_\rho[g(x, y)] - \frac{1}{n}\sum_{i=1}^n g(x^{(i)}, y^{(i)})} & \leq 2\rcomp(\mathcal{F}') + 12R \sqrt{\frac{\log 2/\delta}{2n}},
\end{align*}
with probability at least $1 - \delta$ over the choice of $n$ i.i.d. samples $\qty{(x^{(i)}, y^{(i)})}_{i=1}^n \sim \rho$.
From the contraction lemma~\citep{maurer_vector-contraction_2016}, we have
\begin{align*}
    \rcomp(\mathcal{F}') & = \E_\sigma\qty[\sup_{f \in \mathcal{F}_{2R, 2M}} \sum_{i=1}^n \sigma_i l(f(x_i), y_i)] \\
                         & \leq 2 \sqrt{2R}\rcomp(\mathcal{F}_{2R, 2M})                                            \\
                         & = O\qty(R\sqrt{\frac{(M+1)}{n}}),
\end{align*}
since $l(\cdot, y_i)$ is $2\sqrt{2R}$-Lipschitz continuous in $[-\sqrt{2R}, \sqrt{2R}]$.
Combining above arguments, we arrive at
\begin{align*}
    \bar L(a_{\hat \mu}, \mu) & \leq L(a_{\hat \mu}, \hat \mu) + 2\rcomp(\mathcal{F}') + 12R\sqrt{\frac{\log 2 / \delta}{2n}} \\
                              & = O\qty(\sqrt{\frac{M}{n}} + R\sqrt{\frac{(M+1)}{n}} + R\sqrt{\frac{\log 1/\delta}{n}})       \\
                              & = O\qty((R + 1)\qty(\sqrt{\frac{(M+1)}{n}} + \sqrt{\frac{\log 1/\delta}{n}}))                 \\
                              & = O\qty((R + 1)\sqrt{\frac{(M+1) + \log 1/\delta}{n}}),
\end{align*}
since we set $\lambda = 1/\sqrt{n}$. This completes the proof.

\subsection{Proof of Theorem~\ref{thm:lower-bound}}\label{sec:proof-lower-bound}
Let $\mathcal{S} := \qty{\sin(2\pi u\cdot x) \mid u \in \qty{0, 1}^d, \norm{u}_1 = k}$ be a subset of $L^2(\rho_X)$.
Note that $\mathcal{S}$ is an orthonormal system in $L^2(\rho_X)$.
Assume that
\begin{align*}
    \sup_{f \in \mathcal{B}_M, \norm{f}_{\mathcal{B}_M}^2 \leq R} d(f, H_n) < 1/4.
\end{align*}
Then, from Lemma~\ref{lem:approximation-sin}, for any $\psi = \sin(2\pi u \cdot x) \in \mathcal{S}~(\norm{u}_1 = k)$, there exists $a, \mu$ such that $\kl(\nu \mid \mu) = O(d \log dk + k^2), \norm{a}_{L^2(\mu)} = k$, and
\begin{align*}
    \abs{\psi(x) - f(x)} \leq 1/4
\end{align*}
for any $x \in [0, 1]^d$ since $\abs{u\cdot x} \leq \norm{u}_1\norm{x}_{\infty} \leq k$ and $\sin(2\pi k) = \sin(-2\pi k) = 0$.
Therefore, we have
\begin{align*}
    d(\psi, H_n) & \leq \norm{\psi - f}_{L^2(\rho_X)} + d(f, H_n) \\
                 & < 1/4 + 1/4 = 1/2.
\end{align*}
This contradicts Lemma~\ref{lem:dim-lower-bound}.
Thus, we obtain the result.

\section{Proofs for Section~\ref{sec:feature-learning}}
\subsection{Lemmas for Section~\ref{sec:feature-learning}}
\begin{lemma}\label{lem:l2-functional}
    Assume that $l_i$ is the $l^2$-loss and Assumption~\ref{assumption:data} with $\sigma = 0$.
    Then, we have
    \begin{align*}
        U_{\hat \rho}(\mu) & =  \frac{\bar \lambda_a}{2T} \sum_{i = 1}^T Y_i^\top (\hat \Sigma + n \bar \lambda_a I)^{-1}Y_i,                    \\
        U_{\rho}(\mu)      & =  \frac{\bar \lambda_a}{2T} \sum_{i = 1}^T \langle f^\circ_i, (\Sigma + \bar \lambda_a \id)^{-1}f^\circ_i \rangle.
    \end{align*}
\end{lemma}
\begin{proof}
    From Lemma~\ref{lem:a-l2}, we have
    \begin{align*}
        U_{\hat \rho}(\mu) & = \frac{1}{2nT} \sum_{i=1}^T \norm{Y_i - \hat \Sigma (\hat \Sigma + n \bar \lambda_a I)^{-1} Y_i}_2^2
        + \frac{\bar \lambda_a}{2T} \sum_{i=1}^T Y_i^\top (\hat \Sigma + n\lambda_a I)^{-1} \hat \Sigma (\hat \Sigma + n \bar \lambda_a I)^{-1} Y_i           \\
                           & =  \frac{n\bar \lambda_a^2}{2T} \sum_{i=1}^T \norm{(\hat \Sigma + n \bar \lambda_a I)^{-1}Y_i}_2^2
        + \frac{\bar \lambda_a}{2T} \sum_{i=1}^T Y_i^\top (\hat \Sigma + n \bar \lambda_a I)^{-1} \hat \Sigma (\hat \Sigma + n \bar \lambda_a I)^{-1} Y_i     \\
                           & =  \frac{\bar \lambda_a}{2T} \sum_{i = 1}^T Y_i^\top (\hat \Sigma + n \bar \lambda_a I)^{-1}Y_i,                                 \\
        U_{\rho}(\mu)      & = \frac{1}{2nT} \sum_{i=1}^T \norm{f^\circ_i - \Sigma (\Sigma + \bar \lambda_a \id)^{-1} f^\circ_i}_{L^2(\rho_X)}^2
        + \frac{\bar \lambda_a}{2T} \sum_{i=1}^T \langle f_i^\circ, (\Sigma + \lambda_a \id)^{-1} \Sigma (\Sigma + \bar \lambda_a \id)^{-1} f_i^\circ \rangle \\
                           & =  \frac{\bar \lambda_a^2}{2T} \sum_{i=1}^T \norm{(\Sigma + \bar \lambda_a \id)^{-1}f^\circ_i}_{L^2(\rho_X)}^2
        + \frac{\bar \lambda_a}{2T} \sum_{i=1}^T \langle f_i^\circ, (\Sigma + \lambda_a \id)^{-1} \Sigma (\Sigma + \bar \lambda_a \id)^{-1} f_i^\circ \rangle \\
                           & =  \frac{\bar \lambda_a}{2T} \sum_{i = 1}^T \langle f_i^\circ, (\hat \Sigma + \bar \lambda_a \id)^{-1}f_i^\circ \rangle.
    \end{align*}
\end{proof}

\begin{lemma}\label{lem:lower-bound-A}
    Assume that $l_i$ is the $l^2$-loss and Assumption~\ref{assumption:data} with $\sigma = 0$.
    Let $\hat A'(\mu) = \frac{f^\circ(X)^\top \hat \Sigma_\mu f^\circ(X)}{\norm{f^\circ(X)}_2^2}$
    and $A'(\mu) = \frac{E[f^\circ(x)k(x, x')f^\circ(x')]}{\norm{f^\circ}_{L^2(\rho_X)}^2}$.
    Then, we have
    \begin{align*}
        \hat A(\mu) & \geq \hat A'(\mu) \geq \frac{\bar \lambda_a \norm{f^\circ(X)}_{L^2(\rho)}^2}{2nU_{\hat \rho}(\hat \mu)} - \bar \lambda_a, \\
        A(\mu)      & \geq A'(\mu) \geq \frac{\bar \lambda_a \norm{f^\circ}_{L^2(\rho)}^2}{2U(\mu)} - \bar \lambda_a.
    \end{align*}
\end{lemma}
\begin{proof}
    Since $\norm{\hat \Sigma_\mu} \leq n$ and $E[k(x, x')^2] \leq 1$ from $\abs{h(x;w)} \leq 1$,
    we have $\hat A(\mu) \geq \hat A'(\mu)$ and $A(\mu) \geq A'(\mu)$.
    Let $T = \sum_{i=1}^\infty \mu_i e_i f_i$ be the singular value decomposition of $T$, where $e_m, f_m$ are the orthonormal basis of $L^2(\rho), L^2(\mu)$, respectively.
    Then, there exists $(\alpha_i)_{i \geq 1}$ such that $\sum_i \alpha_i^2 = 1$ and $f^\circ = \norm{f^\circ}_{L^2(\rho)} \alpha_i e_i$.
    Utilizing this expression, we have
    \begin{align*}
        U_{\rho}(\mu) & = \frac{\bar \lambda_a}{2}\norm{f^\circ}_{L^2(\rho)}^2 \sum_{i} \frac{\alpha_i^2}{\mu_i + \bar \lambda_a}     \\
                      & \geq \frac{\bar \lambda_a}{2}\norm{f^\circ}_{L^2(\rho)}^2 \frac{1}{\sum_i \alpha_i^2(\mu_i + \bar \lambda_a)} \\
                      & \geq \frac{\bar \lambda_a}{2}\norm{f^\circ}_{L^2(\rho)}^2 \frac{1}{A'(\mu) + \bar \lambda_a}.
    \end{align*}
    The second inequality follows from the convexity of $1/x$ and Jensen's inequality.
    By the same argument, we have
    \begin{align*}
        U_{\rho}(\mu) & \geq \frac{\bar \lambda_a}{2}\norm{f^\circ}_{L^2(\rho)}^2 \frac{1}{n\hat A'(\mu) + n\bar \lambda_a}.
    \end{align*}
    By transposition of the above inequalities, we obtain the results.
\end{proof}

\begin{lemma}\label{lem:intrinsic-noise}
    Assume that $l_i$ is the $l^2$-loss and Assumption~\ref{assumption:data} holds.
    Then, we have
    \begin{align*}
        \E_\varepsilon [U_{\hat \rho}(\mu)] & =\frac{\bar \lambda_a}{2T} \sum_{i=1}^T f_i^\circ(X)^\top (\hat \Sigma + n \bar \lambda_a I)^{-1} f_i^\circ(X) + \frac{\bar \lambda_a \sigma^2}{6n} (n - d_{\bar\lambda_a}(\mu)).
    \end{align*}
\end{lemma}
\begin{proof}
    \begin{align*}
        \E_\varepsilon [U_{\hat \rho}(\mu)]
         & = \frac{\bar \lambda_a}{2T} \sum_{i=1}^T \E_\varepsilon[(f_i^\circ(X) + \varepsilon_i)^\top (\hat \Sigma + n \bar \lambda_a I)^{-1} (f_i^\circ(X) + \varepsilon_i)]                             \\
         & =\frac{\bar \lambda_a}{2T} \sum_{i=1}^T f_i^\circ(X)^\top (\hat \Sigma + n \bar \lambda_a I)^{-1} f_i^\circ(X) + \frac{\bar \lambda_a \sigma^2}{6} \tr((\hat \Sigma + n \bar \lambda_a I)^{-1}) \\
         & =\frac{\bar \lambda_a}{2T} \sum_{i=1}^T f_i^\circ(X)^\top (\hat \Sigma + n \bar \lambda_a I)^{-1} f_i^\circ(X) + \frac{\bar \lambda_a \sigma^2}{6n} (n - d_{\bar\lambda_a}(\mu)).
    \end{align*}
\end{proof}

\subsection{Proof of Theorem~\ref{thm:alignment}}\label{sec:proof-alignment}
For $r > 0$, let
\begin{align*}
    \tilde f_r(z) = \begin{cases}
                        \tilde f(z)                               & \text{if } \abs{z} \leq r                         \\
                        \tilde f(r) - \sgn(\tilde f(r)) (z - r)   & \text{if } r \leq z \leq r + \abs{\tilde f(r)}    \\
                        \tilde f(-r) - \sgn(\tilde f(-r)) (r - z) & \text{if } -r - \abs{\tilde f(-r)} \leq z \leq -r \\
                        0                                         & \text{otherwise}
                    \end{cases}.
\end{align*}
This is continuous, differentiable almost everywhere and its derivative satisfies $\abs{\tilde f_r'(x)} \leq 1$ a.s.
From Lemma~\ref{lem:approximation-sin}, there exists $f := \int a(w)h(x;w)\dd \mu(w)$ such that $\kl(\nu \mid \mu) = O(r^2 / \varepsilon^2 + d \log \frac{d r r_x}{\varepsilon}), \norm{a}_{L^2(\mu)} = O(r)$ and
\begin{align*}
    \abs{\tilde f_r(u^\circ \cdot x) - f(x)} & \leq \varepsilon
\end{align*}
for any $x \in \R^d$ such that $\abs{u^\circ \cdot x} \leq r, \norm{x}_2 \leq r_x \sqrt{d}$.
Since $f^\circ(x) = \tilde f(u^\circ \cdot x) = \tilde f_r(u^\circ \cdot x)$ for $\abs{u^\circ \cdot x} \leq r$, we have
\begin{align*}
    \abs{f^\circ(x) - f(x)} & \leq \varepsilon
\end{align*}
for any $x \in \R^d$ such that $\norm{x} \leq r_x \sqrt{d}$ and $\abs{u \cdot x} \leq r$.
From the tail bound on Gaussian and chi-squared distribution~\citep{wainwright2019high}, we have
\begin{align*}
    P\qty(\frac{1}{d} \norm{x}^2 \geq \frac{r_x^2}{d}, \abs{u \cdot x} \geq r) \leq 2\exp(-r^2 / 2) + 2\exp(-d(r_x^2 - 1) / 8).
\end{align*}
Thus, $\norm{f^\circ - f}_{L^2(\rho_X)}$ is evaluated as follows:
\begin{align*}
    \norm{f^\circ - f}_{L^2(\rho_X)}^2 & \leq (\norm{f^\circ}_\infty + \norm{f}_{\infty})^2P(\norm{x} \geq \sqrt{d}r_x, \abs{w \cdot x} \geq r) + \int_{\norm{x} \leq \sqrt{d}, \abs{w \cdot x} \leq r} (f^\circ(x) - f(x))^2 \dd \rho_X(x) \\
                                       & \leq (\norm{a}_{L^2(\mu)} + 1)^2P\qty(\frac{1}{n} \norm{x}^2 \geq \frac{r_x^2}{n}, \abs{u \cdot x} \geq r) + \varepsilon^2                                                                         \\
                                       & \leq 1/16 := \bar \varepsilon^2
\end{align*}
by setting $\varepsilon$ to a sufficiently small constant and $r, r_x$ sufficiently large constants which are independent of $d$.
Thus, there exists $M = O(d \log d), R = O(1)$ such that $\norm{f^\circ - f}_{L^2(\rho_X)}^2$, where $f \in \mathcal{F}_{R, M}$.
By the same argument in the proof of Theorem~\ref{thm:generalization-bound}, we have
\begin{align*}
    L_{\hat \rho}(a, \mu) & \leq L_\rho(a, \mu) + O\qty(R\sqrt{\frac{(M + 1) + \log 1/\delta}{n}})
\end{align*}
with probability at least $1 - \delta$ over the choice of training data.
Thus, by setting $n \geq c_3'(d \log d + \log 1/\delta)$ for sufficiently large $c_3'$, we have
\begin{align*}
    L_{\hat \rho}(a, \mu) & \leq L_\rho(a, \mu) + \bar \varepsilon^2 \leq 2\bar \varepsilon^2
\end{align*}

From the optimality of $\hat \mu$ and $\hat a = a_{\hat \mu}$, we have
\begin{align*}
    \mathcal{G}_{\hat \rho}(\hat \mu) & = L(\hat a, \hat \mu) + \frac{\bar \lambda_a}{2} \norm{\hat a}_{L^2(\hat \mu)}^2 + \lambda \kl(\nu \mid \hat \mu) \\
                                      & \leq \mathcal{G}(a, \mu)                                                                                          \\
                                      & \leq 2\tilde \varepsilon^2 + 2 \lambda M                                                                          \\
                                      & \leq 3\tilde \varepsilon^2,
\end{align*}
by setting $\lambda = \bar \epsilon^2 / 2M$ and $\lambda_a = M / R$.
Thus, it holds that
\begin{align*}
    \norm{\hat a}_{L^2(\hat \mu)}^2 & \leq 12R, \\
    \kl(\nu  \mid \hat \mu)         & \leq 6M.
\end{align*}
Then, by the same reasoning as in the proof of Theorem~\ref{thm:generalization-bound}, we have
\begin{align*}
    U_\rho(\hat \mu) & \leq E_\rho[l(f(x;\hat \mu), y)] + \frac{\bar \lambda_a}{2} \norm{a_{\hat \mu}}_{L^2(\mu)}^2                                                               \\
                     & \leq E_{\hat \rho}[l(f(x;\hat \mu), y)] + \frac{\bar \lambda_a}{2} \norm{a_{\hat \mu}}_{L^2(\mu)}^2 + O\qty((R + 1)\sqrt{\frac{M + 1 + \log 1/\delta}{n}}) \\
                     & \leq \mathcal{G}_{\hat \rho}(\hat \mu) + O\qty((R + 1)\sqrt{\frac{M + 1 + \log 1/\delta}{n}})                                                              \\
                     & \leq 3 \bar \varepsilon^2 + O\qty((R + 1)\sqrt{\frac{M + 1 + \log 1/\delta}{n}})                                                                           \\
                     & \leq 4\bar \varepsilon^2
\end{align*}
by setting $n = c_3(d \log d + \log 1/\delta)$ for a sufficiently large constant $c_3 \geq c_3'$
with probability at least $1 - \delta$ over the choice of training data.
Therefore, it holds that
\begin{align*}
    A(\hat \mu) \geq A'(\hat \mu) & \geq \frac{\bar \lambda_a }{2U(\hat \mu)} - \bar \lambda_a                          \\
                                  & \geq \frac{\bar \varepsilon^2 / (2R)}{8 \bar \varepsilon^2} - \bar \varepsilon^2/2R \\
                                  & \geq \frac{1 - 8\bar \epsilon^2}{16R}                                               \\
                                  & = \Omega(1).
\end{align*}

Let $\qty{u_i}_{i=1}^d~(u_1 = u)$ be an orthonormal basis of $\R^d$.
Then, the symmetry of $\mu_0 = \nu$ implies that $\int \frac{\langle u_i, w\rangle^2}{\norm{w}_2^2} \dd \mu_0(w)$ is equal for any $i$.
Since $\sum_{i=1}^d \int \frac{\langle u_i, w\rangle^2}{\norm{w}_2^2} \dd \mu_0(w) = 1$, we have $\int \frac{\langle u, w\rangle^2}{\norm{w}_2^2} \dd \mu_0(w) = 1 / d$.
In addition let $f_i(x) = \tilde f(u_i \cdot x)$.
Then, we have $\E_\rho[f_i(x)f_j(x)] = 0$ for $i\neq j$.
Thus, we have
\begin{align*}
    \E_{\rho_X}\qty[\sum_{i=1}^d f_i(x)k(x, x')f_i(x')] & \leq \sqrt{\E_{\rho_X}\qty[\qty(\sum_{i=1}^d f_i(x)f_i(x'))^2]} \sqrt{\E_{\rho}[k(x, x')^2]} \\
                                                        & = \sqrt{\sum_{i=1}^d \E_{\rho}\qty[f_i(x)^2 f_i(x')^2]} \sqrt{\E_{\rho}[k(x, x')^2]}         \\
                                                        & = \sqrt{d} \norm{f^\circ(x)}_{L^2(\rho_X)}^2 \sqrt{\E_{\rho}[k(x, x')^2]}.
\end{align*}
Since $\rho_X$ and $\mu_0$ are rotationally invariant, it holds that
\begin{align*}
    \E_{\rho}\qty[\sum_{i=1}^d f_i(x)k_{\mu_0}(x, x')f_i(x')] & = d \E_{\rho}\qty[f_1(x)k_{\mu_0}(x, x')f_1(x')]           \\
                                                              & = d \E_{\rho}\qty[f^\circ(x) k_{\mu_0}(x, x')f^\circ(x')].
\end{align*}
Thus, the kernel alignment at the initialization can be evaluated as follows:
\begin{align*}
    A(\mu_0) & = \frac{\E_{\rho}\qty[f^\circ(x) k_{\mu_0}(x, x')f^\circ(x')]}{\norm{f^\circ(x)}_{L^2(\rho_X)}^2 \sqrt{\E_{\rho}[k(x, x')^2]}} \\
             & \leq \frac{1}{\sqrt{d}}.
\end{align*}

Let $u = u^\parallel + u^{\perp}$ and $x = x^\parallel + x^\perp$ be the orthogonal decomposition of $u$ and $x$ with respect to the space spanned by the rows of $u^\circ$.
Then, we have
\begin{align*}
    \E_\rho[yh(x;w)] & = \int \tilde f(u^\circ \cdot x) \tilde h(x \cdot u + b) \dd \rho_X(x)                                       \\
                     & = \int \tilde f(u^\circ \cdot x^\parallel)\tilde h(u \cdot x^\parallel + u \cdot x^\perp + b) \dd \rho_X(x).
\end{align*}
Here, $x^\parallel$ and $x^\perp$ follows the normal distribution $N(0, I_k), N(0, I_{d-k})$ independently.
In addition, $u \cdot x^\perp$ follows $N(0, \norm{u^{\perp}})$.
Therefore, we have
\begin{align*}
    \E_\rho[yh(x;w)] & = \int \tilde f(u^\circ \cdot x^\parallel)\tilde h_{\norm{u^\perp}}(u^\parallel \cdot x^\parallel + \tau) \dd \nu_k(x^\parallel).
\end{align*}
where $\tilde h_\tau(x) = \E_{z\sim \nu}[\tilde h(x + \tau z)]$.
For $\tilde h_\tau$, we have the following lemma.
\begin{lemma}
    For any $\tau \in \R$, we have the following results
    \begin{itemize}
        \item $\tilde h_\tau$ is $L_\tau$-Lipschitz continuous with $L_\tau = \min\qty{1, \frac{4}{\pi \tau}}$.
    \end{itemize}
\end{lemma}
\begin{proof}
    From the definition of $\tilde h_\tau$, we have
    \begin{align*}
        \tilde h_\tau'(x) & = \E_{z \sim N(0, 1)}[\tanh'(x + \tau z)]                                                  \\
                          & = \frac{1}{\sqrt{2\pi}} \int_{-\infty}^\infty \tanh'(x + \tau z)e^{-z^2/2} \dd z           \\
                          & = \frac{1}{\sqrt{2\pi}} \int_{-\infty}^\infty \tanh'(\tau z)e^{-(z - x / \tau)^2/2} \dd z.
    \end{align*}
    Thus, $\tilde h_\tau''(x)$ is given by
    \begin{align*}
        \tilde h_\tau''(x)
         & = \frac{1}{\sqrt{2\pi}\tau} \int_{-\infty}^\infty (z - x / \tau) \tanh'(\tau z)e^{-(z - x / \tau)^2/2} \dd z                     \\
         & = \frac{1}{\sqrt{2\pi}\tau} \int_{-\infty}^\infty z \tanh'(x + \tau z)e^{-z^2/2} \dd z                                           \\
         & = \frac{1}{\sqrt{2\pi}\tau} \int_0^\infty \frac{ze^{-z^2/2}}{\cosh^2(x + \tau z)} -\frac{ze^{-z^2/2}}{\cosh^2(x - \tau z)}\dd z.
    \end{align*}
    From the symmetry of $\cosh$, we have
    \begin{align*}
        \frac{1}{\cosh^2(x + \tau z)} -\frac{1}{\cosh^2(x - \tau z)} & \leq 0
    \end{align*}
    for any $x \geq 0$. Thus, we have $\tilde h_\tau''(x) \leq 0$ for any $x \geq 0$, that is, $\tilde h_\tau(x)$ is monotonically decreasing in $[0, \infty]$.
    From the symmetry of $\tilde h_\tau'(x)$ and $\tilde h_\tau'(x) \geq 0$ for any $x$,
    $\tilde h_\tau(0) \geq \abs{\tilde h_\tau(x)}$ for any $x$.
    Thus, it suffices to evaluate $\tilde h_\tau'(0)$.
    Here, we have
    \begin{align*}
        \tilde h_\tau'(0) & = \E_{z\sim \nu}[\tilde h'(bz)]                                                                         \\
                          & = \frac{1}{\sqrt{2\pi}} \int_{-\infty}^\infty \frac{4}{(e^{bz} + e^{-bz})^2} \exp(-\frac{z^2}{2}) \dd z \\
                          & \leq \frac{8}{\sqrt{2\pi}} \int_{0}^\infty e^{-2bz} \exp(-\frac{z^2}{2}) \dd z                          \\
                          & \leq \frac{8}{\sqrt{2\pi}} e^{2\tau^2}\int_{0}^\infty \exp(-\frac{(z + 2\tau)^2}{2}) \dd z              \\
                          & \leq \frac{8}{\sqrt{2\pi}} e^{2\tau^2}\int_{2\tau}^\infty \exp(-\frac{z'^2}{2}) \dd z'                  \\
                          & \leq \frac{8}{\sqrt{2\pi}} e^{2\tau^2}\int_{2\tau}^\infty \exp(-\frac{z'^2}{2}) \dd z'                  \\
                          & \leq \frac{8}{\sqrt{2\pi}} e^{2\tau^2} P(Z' > 2\tau)                                                    \\
                          & \leq \frac{8}{\sqrt{2\pi}} e^{2\tau^2} \sqrt{\frac{1}{2\pi}} \frac{e^{-2\tau^2}}{2\tau}                 \\
                          & = \frac{2}{\pi \tau}.
    \end{align*}
    For the last inequality, we used Mill's inequality. Obviously, $\tilde h_\tau'(0) \leq 1$.
    This completes the proof.
\end{proof}

Utilizing the above lemma, we have
\begin{align*}
    \abs{\E_\rho[f^\circ(x) h(x;w)]}^2 & \leq \abs{\int \tilde f(u^\circ \cdot x^\parallel)\tilde h_{\norm{u^\perp}}(u^\parallel \cdot x^\parallel + \tau) \dd \nu_k(x^\parallel)}^2                                         \\
                                       & \leq \abs{\int \tilde f(u^\circ \cdot x^\parallel)\qty[\tilde h_\tau(\tau) + (\tilde h_\tau(u^\parallel \cdot x^\parallel + \tau) - \tilde h_\tau(\tau))] \dd \nu_k(x^\parallel)}^2 \\
                                       & \leq \abs{\int \tilde f(u^\circ \cdot x^\parallel)(\tilde h_\tau(u^\parallel \cdot x^\parallel + \tau) - \tilde h_\tau(\tau)) \dd \nu_k(x^\parallel)}^2                             \\
                                       & \leq \norm{f^\circ}_{L^2(\rho_X)}^2 \int (\tilde h_\tau(u^\parallel \cdot x^\parallel + \tau) - \tilde h_\tau(\tau))^2 \dd\nu(z_1)                                                  \\
                                       & \leq \norm{f^\circ}_{L^2(\rho_X)}^2 \int (L_\tau u^\parallel \cdot x^\parallel)^2 \dd\nu_k(x^\parallel)                                                                             \\
                                       & \leq \norm{f^\circ}_{L^2(\rho_X)}^2 \int \qty(L_\tau \norm{u^\parallel} z)^2 \dd\nu(z)                                                                                              \\
                                       & \leq \norm{f^\circ}_{L^2(\rho_X)}^2 L_\tau^2 \norm{u^\parallel}^2                                                                                                                   \\
                                       & \leq \norm{f^\circ}_{L^2(\rho_X)}^2 \frac{8\norm{u^\parallel}^2}{\pi^2 \norm{u^\perp}^2}.
\end{align*}
From the boundedness of $h$, we have $\abs{\E_\rho[f^\circ(x)h(x;w)]}^2 \leq \norm{f^\circ}_{L^2(\rho_X)}^2$.
Combining above arguments, we have
\begin{align*}
    \abs{\E_\rho[f^\circ(x)h(x;w)]}^2 & \leq \norm{f^\circ}_{L^2(\rho_X)}^2 \min\qty{1, \frac{8\norm{w^\parallel}^2}{\pi^2 \norm{w^\perp}^2}} \\
\end{align*}
By the way, for any $a, b$, we have
\begin{align*}
    \min\qty{1, \frac{a^2}{b^2}} \leq 2 \frac{a^2}{a^2 + b^2}
\end{align*}
Therefore, it holds that
\begin{align*}
    \E_\rho[f^\circ(x)h(x;w)]^2 & \leq \norm{f^\circ}_{L^2(\rho_X)} \frac{16}{\pi^2} \frac{\norm{w^\parallel}^2}{\norm{w}^2}
\end{align*}
Recall that $A'(\mu) = \frac{E_\rho[f^\circ(x) E_\mu[h(x;w)h(x';w)] f^\circ(x')]}{\norm{f^\circ}_{L^2(\rho)}^2} = \frac{E_\mu[E_{\rho}[f^\circ (x)h(x;w)]^2]}{\norm{f^\circ}_{L^2(\rho)}^2}$.
Thus, from the definition of $P(\mu)$, we have
\begin{align*}
    P(\hat \mu) & = \E\qty[\frac{\norm{w^\parallel}^2}{\norm{w}^2}] \\
                & = \Omega(A'(\hat \mu)) = \Omega(1).
\end{align*}

\subsection{Proof of Theorem~\ref{thm:label-noise}}\label{sec:proof-label-noise}
Tha label noise procedure can be regarded as a SGD-MFLD in~\cite{suzuki_convergence_2023}.
Thus, as in the proof of Theorem~\ref{thm:convergence}, we follow the framework in~\citet{suzuki_convergence_2023}.

First, we show the convexity of $\bar U(\mu):= E_{\tilde \varepsilon}[U_{\bar \epsilon}(\mu)]$.
The functional $U_{\tilde \varepsilon}(\mu)$ can be written as
\begin{align*}
    U_{\tilde \varepsilon}(\mu) & = \frac{1}{T} \sum_{i=1}^T \qty[\frac{1}{2n} \norm{Y_i - \hat \Sigma_{\mu} (\hat \Sigma_{\mu} + n \bar \lambda_a I)^{-1} (Y_i + \tilde \varepsilon_i)}_2^2 + \frac{\bar \lambda_a}{2} (Y_i + \tilde \varepsilon_i)^\top (\hat \Sigma_{\mu} + n\bar \lambda_a I)^{-1}\hat \Sigma_{\mu} (\hat \Sigma_{\mu} + n \bar \lambda_a I)^{-1} (Y_i + \tilde \varepsilon_i)] \\
                                & = \frac{1}{T} \sum_{i=1}^T \qty[\frac{\bar \lambda_a}{2} Y_i^\top (\hat \Sigma_{\mu} + n\bar \lambda_a I)^{-1} Y_i + \frac{1}{2n} \tilde \varepsilon_i^\top (\hat \Sigma_{\mu} + n\bar \lambda_a I)^{-1}\hat \Sigma_{\mu} \tilde \varepsilon_i]                                                                                                                   \\
                                & = \frac{1}{T} \sum_{i=1}^T \qty[\frac{\bar \lambda_a}{2} Y_i^\top (\hat \Sigma_{\mu} + n\bar \lambda_a I)^{-1} Y_i - \frac{\bar \lambda_a}{2} \tilde \varepsilon_i^\top (\hat \Sigma_{\mu} + n\bar \lambda_a I)^{-1} \tilde \varepsilon_i + \frac{\norm{\tilde \varepsilon_i}^2}{2n}] .                                                                           \\
\end{align*}
Taking the expectation of the above equation with respect to $\tilde \varepsilon_i$, we have
\begin{align*}
    \E_{\varepsilon}[U_{\tilde \varepsilon}(\mu)] & = \frac{1}{T} \sum_{i=1}^T \frac{\bar \lambda_a}{2} Y_i^\top (\hat \Sigma_{\mu} + n\bar \lambda_a I)^{-1} Y_i - \frac{\tilde \sigma^2 \bar \lambda_a}{6} \tr[(\hat \Sigma_{\mu} + n\bar \lambda_a I)^{-1}] + \frac{\tilde \sigma^2}{6} \\
                                                  & = \frac{\bar \lambda_a}{6} \tr[(\hat \Sigma_{\mu} + n \bar \lambda_a)^{-1} \tilde Y] + \frac{\tilde \sigma^2}{6},
\end{align*}
where $\tilde Y = \frac{1}{T}\sum_{i=1}^T Y_i Y_i^\top - \tilde \sigma^2 I / 3$.
From the assumption on $\tilde \sigma^2$, $\tilde Y$ is positive semi-definite and thus, $\bar U(\mu)$ is convex.

Next, we derive an LSI constant for $\tilde p_{\mu} \propto \exp(-\frac{1}{\lambda}\fdv{\bar F(\mu)}{\mu})$.
The first variation of $\bar F(\mu)$ is given by
\begin{align*}
    \fdv{\bar F(\mu)}{\mu}\/(w) & = -\frac{\bar \lambda_a}{2}\tr[(\hat \Sigma_{\mu} + n\bar \lambda_a)^{-1}h(X;w)h(X;w)^\top (\hat \Sigma_{\mu} + n\bar \lambda_a I)^{-1}\tilde Y] + \frac{\bar \lambda_w}{2} \norm{w}_2^2.
\end{align*}
Since $\norm{\tilde Y}_2 \leq \frac{1}{T} \sum_{i=1}^T \norm{Y_i}_2 \leq n c_l^2$, we have
\begin{align*}
    \abs{\frac{\bar \lambda_a}{2}\tr[(\hat \Sigma_{\mu} + n\bar \lambda_a)^{-1}h(X;w)h(X;w)^\top (\hat \Sigma_{\mu} + n\bar \lambda_a I)^{-1}\tilde Y]} & \leq \frac{\norm{\tilde Y}_2}{2n^2\bar \lambda_a}\norm{h(X;w)}_2^2 \leq \frac{c_l^2}{\bar \lambda_a}.
\end{align*}
Thus, $\bar p_\mu$ satisfies the LSI with the same constant as in Lemma~\ref{lem:lsi}.

Let $V(\mu) := \frac{1}{2T} \sum_{i=1}^T \tilde \varepsilon_i^\top (\hat \Sigma_{\mu} + n\bar \lambda_a I)^{-1}\tilde \varepsilon_i$.
Then, $V(\mu)$ is equal to $U(\mu)$ if $Y_i = \varepsilon_i$ for any $i \in [T]$.
Thus, by the same argument as in the proof of Theorem~\ref{thm:convergence}, we have
\begin{itemize}
    \item $\norm{\grad \fdv{V}{\mu}\/(\mu)(w) - \grad \fdv{V}{\mu}\/(\mu')(w')} \leq L_V (W_2(\mu, \mu') + \norm{w - w'}), \abs{\fdv[2]{U(\mu)}{\mu}\/(w, w')} \leq L_V$ for $L_U = 4R_a^2 + \bar \lambda_a(B_a'L_a' + R_a'^2) + B_a'^2$.
    \item $\norm{\grad \fdv{V}{\mu}\/(\mu)(w)} \leq R_V$ for $R_V = \bar \lambda_a B_a' R_a'$.
\end{itemize}
where we define $R_a', B_a', L_a'$ by replacing $c_l$ in $R_a, B_a, L_a$ with $\tilde \sigma$, respectively.
Furthermore, it holds that
\begin{itemize}
    \item $\norm{\grad_w \fdv{V}{\mu}\/(\mu)(w)}_2 \leq R_V$.
    \item $\norm{\grad_w \grad_w^\top \fdv{V}{\mu}\/(\mu)(w)}_{\ope} \leq L_V$, $\norm{\grad_w \grad_{w'}^\top \grad_w^\top \fdv[2]{V}{\mu}\/(\mu)(w, w')}_{\ope} \leq L_V$.
\end{itemize}
Let $L_{UV} = L_{U} + L_V, R_{UV} = R_U + R_V$.
Then, it holds That
\begin{itemize}
    \item $\norm{\grad \fdv{U_{\tilde \varepsilon}}{\mu}\/(\mu)(w) - \grad \fdv{V}{\mu}\/(\mu')(w')} \leq L_{UV} (W_2(\mu, \mu') + \norm{w - w'}), \abs{\fdv[2]{U(\mu)}{\mu}\/(w, w')} \leq L_{UV}$,
    \item $\norm{\grad \fdv{U_{\tilde \varepsilon}}{\mu}\/(\mu)(w)} \leq R_{UV}$,
    \item $\norm{\grad_w \fdv{U_{\tilde \varepsilon}}{\mu}\/(\mu)(w)}_2 \leq R_{UV}$,
    \item $\norm{\grad_w \grad_w^\top \fdv{U_{\tilde \varepsilon}}{\mu}\/(\mu)(w)}_{\ope} \leq L_{UV}$, $\norm{\grad_w \grad_{w'}^\top \fdv[2]{U_{\tilde \varepsilon}}{\mu}\/(\mu)(w, w')}_{\ope} \leq L_{UV}$.
\end{itemize}
since $U_{\tilde \varepsilon}(\mu) = U(\mu) - V(\mu) + \text{const.}$
Combining above arguments, Theorem 3 in~\citet{suzuki_convergence_2023} yields
\begin{align*}
    \frac{1}{N} \E[\mathcal{L}^N(\mu_k^{(N)})] - \mathcal{L}(\mu^*) \leq \exp(-\lambda \alpha \eta k) \qty(\frac{1}{N} \E[\mathcal{L}^N(\mu_k^{(N)})] - \mathcal{L}(\mu^*)) + \frac{2}{\lambda \alpha}\bar L^2C_1(\lambda \eta + \eta^2) + \frac{4}{\lambda \alpha \eta} \bar {\Upsilon} + \frac{2C_\lambda}{\lambda \alpha N},
\end{align*}
where $\bar R^2 = \E\qty[\norm{w_i^{(0)}}_2] + \frac{1}{\bar \lambda_w}\qty[\qty(\frac{1}{4}+\frac{1}{\bar \lambda_w})R_{UV}^2 + \lambda d']$,
$\bar L = L_{UV} + \bar \lambda_w$, $C_1 = 8[R_{UV}^2 + \bar \lambda_w \bar R^2 + d']$,
$C_\lambda = 2\lambda L_{UV} \alpha + 2\lambda^2 L_{UV}^2\bar R^2$,
and
\begin{align*}
    \bar \Upsilon := 4\eta\delta_\eta + [R_{UV} + \lambda_w \bar R + (L_{UV} + \bar \lambda_w)^2](1 + \sqrt{\lambda / \eta})\eta^2 R_{UV}^2 + (R_{UV} + \bar \lambda_w \bar R)R_{UV}(1 + \sqrt{\lambda / \eta})\eta^3 R_{UV}^2.
\end{align*}
This completes the proof.
\end{document}